\documentclass[sigconf]{acmart}
\usepackage{microtype}
\usepackage{graphicx}
\usepackage{subfigure}
\usepackage{booktabs} 
\usepackage{enumitem}
\usepackage{xcolor}
\usepackage{multirow}
\usepackage[table]{xcolor}
\usepackage{amsmath}       
\usepackage{amsfonts}     
\usepackage[most]{tcolorbox}

\definecolor{algocyan}{RGB}{0, 150, 150}
\definecolor{beaublue}{rgb}{0.74, 0.83, 0.9}

\usepackage{amsmath} 
\usepackage{mathtools}
\usepackage{amsthm}
\usepackage{algorithmic}
\usepackage{algorithm}
\usepackage[capitalize,noabbrev]{cleveref}

\theoremstyle{plain}
\newtheorem{theorem}{Theorem}[section]

\theoremstyle{definition}

\theoremstyle{remark}

\usepackage[textsize=tiny]{todonotes}
\AtBeginDocument{%
  }

\setcopyright{acmlicensed}
\copyrightyear{2026}
\acmYear{2026}
\setcopyright{cc}
\setcctype{by}
\acmConference[KDD '26]{Proceedings of the 32nd ACM SIGKDD Conference on Knowledge Discovery and Data Mining V.2}{August 09--13, 2026}{Jeju Island, Republic of Korea}
\acmBooktitle{Proceedings of the 32nd ACM SIGKDD Conference on Knowledge Discovery and Data Mining V.2 (KDD '26), August 09--13, 2026, Jeju Island, Republic of Korea}
\acmDOI{10.1145/3770855.3817617}
\acmISBN{979-8-4007-2259-2/2026/08}
\definecolor{softblue}{RGB}{240, 244, 250}
\definecolor{softsage}{RGB}{242, 248, 245}
\definecolor{highlightgold}{RGB}{255, 250, 240}
\begin{document}

%%
%% The "title" command has an optional parameter,
%% allowing the author to define a "short title" to be used in page headers.
\title{LC-ERD: Mining Latent Logic for Self-Evolving Reasoning via Consistency-Regulated Reward Decomposition}

%%
%% The "author" command and its associated commands are used to define
%% the authors and their affiliations.
%% Of note is the shared affiliation of the first two authors, and the
%% "authornote" and "authornotemark" commands
%% used to denote shared contribution to the research.
\author{Yanyu Chen}
\authornote{Both authors contributed equally to this research.}
\affiliation{%
  \institution{The Chinese University of Hong Kong}
  \city{Hong Kong SAR}
  \country{China}}
\email{chenyanyu.cse@link.cuhk.edu.hk}

\author{Jiyue Jiang}
\authornotemark[1]
\affiliation{%
  \institution{The Chinese University of Hong Kong}
  \city{Hong Kong SAR}
  \country{China}}
\email{jiangjy@link.cuhk.edu.hk}

\author{Dianzhi Yu}
\affiliation{%
  \institution{The Chinese University of Hong Kong}
  \city{Hong Kong SAR}
  \country{China}}
\email{dianzhi.yu@link.cuhk.edu.hk}

\author{Zheng Wu}
\affiliation{%
  \institution{Shanghai Jiaotong University}
  \city{Shanghai}
  \country{China}}
\email{wzh815918208@sjtu.edu.cn}

\author{Jiahong Liu}
\affiliation{%
  \institution{The Chinese University of Hong Kong}
  \city{Hong Kong SAR}
  \country{China}}
\email{jiahong.21@gmail.com}

\author{Jiaming Han}
\affiliation{%
  \institution{The Chinese University of Hong Kong}
  \city{Hong Kong SAR}
  \country{China}}
\email{hanjiaming@link.cuhk.edu.hk}

\author{Xiao Guo}
\affiliation{%
  \institution{Fudan University}
  \city{Shanghai}
  \country{China}}
\email{guox26@mail2.sysu.edu.cn}

\author{Jinhu Qi}
\affiliation{%
  \institution{The Chinese University of Hong Kong}
  \city{Hong Kong SAR}
  \country{China}}
\email{jinhuqi@link.cuhk.edu.hk}

\author{Yu Li}
\affiliation{%
  \institution{The Chinese University of Hong Kong}
  \city{Hong Kong SAR}
  \country{China}}
\email{liyu@cse.cuhk.edu.hk}

\author{Yifei Zhang}
\affiliation{%
  \institution{The Chinese University of Hong Kong}
  \city{Hong Kong SAR}
  \country{China}}
\email{yifeiacc@gmail.com}

\author{Irwin King}
\authornote{Corresponding author.}
\affiliation{%
  \institution{The Chinese University of Hong Kong}
  \city{Hong Kong SAR}
  \country{China}}
\email{king@cse.cuhk.edu.hk}

%%
%% By default, the full list of authors will be used in the page
%% headers. Often, this list is too long, and will overlap
%% other information printed in the page headers. This command allows
%% the author to define a more concise list
%% of authors' names for this purpose.
\renewcommand{\shortauthors}{Chen et al.}

%%
%% The abstract is a short summary of the work to be presented in the
%% article.
% \begin{abstract}
% The self-alignment of Large Language Model (LLM) reasoning is a critical challenge, as conventional endogenous rewards often fail to capture the structural nuances of complex logic. Current self-evolution faces three primary challenges: (1) a \textit{mimetic bias} where rewards prioritize statistical likelihood over logical truth, creating a ``correctness illusion'' that masks compounding reasoning errors; (2) a \textit{credit assignment bottleneck} where sparse global outcomes fail to provide granular guidance for intermediate steps; and (3) the \textit{failure of endogenous signals} to generalize across open-domain tasks without amplifying pre-training biases. To address these gaps, we introduce \textbf{LC-ERD} (Logic-Consistent Endogenous Reward Decomposition), a framework that holistically aligns problem-solving trajectories through variational consensus. By deriving a \textit{Variational Logic Potential} from \textit{Latent Logic Expertise} (LLE), we quantify structural soundness alongside probability to counter mimetic bias. To resolve credit assignment, LC-ERD introduces a \textit{Multi-Agent Value Decomposition} protocol based on the IGM principle, quantifying individual step utility within the joint reasoning process. Experiments show LC-ERD delivers a robust self-evolution path that uncovers critical trade-offs between logic consistency and accuracy entirely missed by standard endogenous rewards.
% \end{abstract}
\begin{abstract}
The evolution of Large Language Model (LLM) reasoning is bottlenecked by the scarcity of high-quality process data. While self-alignment via endogenous rewards offers a solution, mining valid supervision faces three challenges: (1) \textit{Label Noise via Mimetic Bias}, where rewards prioritize statistical likelihood over logical truth, creating a ``correctness illusion'' that masks compounding errors; (2) \textit{Coarse-Grained Supervision}, where sparse global outcomes (e.g., in GRPO) fail to provide granular guidance, treating reasoning chains as monolithic; and (3) \textit{Distributional Collapse}, where signals fail to generalize without amplifying pre-training biases. To address these, we introduce \textbf{LC-ERD} (Logic-Consistent Endogenous Reward Decomposition), a framework framing self-alignment as latent structure mining. We derive a \textit{Variational Logic Potential} by aggregating consensus from the model's \textit{Latent Logic Expertise} (LLE) to denoise the reasoning manifold, and introduce a \textit{Multi-Agent Value Decomposition} protocol based on the IGM principle to quantify individual step utility. Experiments show LC-ERD delivers a robust self-evolution path, uncovering trade-offs between logic consistency and accuracy while identifying high-value reasoning patterns missed by standard rewards. Our code is available at \url{https://github.com/LC-ERD-repo/LC-ERD}.
\end{abstract}

%%
%% The code below is generated by the tool at http://dl.acm.org/ccs.cfm.
%% Please copy and paste the code instead of the example below.
%%
\begin{CCSXML}
<ccs2012>
   <concept>
       <concept_id>10010147.10010257.10010258.10010261</concept_id>
       <concept_desc>Computing methodologies~Reinforcement learning</concept_desc>
       <concept_significance>500</concept_significance>
       </concept>
 </ccs2012>
\end{CCSXML}

\ccsdesc[500]{Computing methodologies~Reinforcement learning}

% \begin{CCSXML}
% <ccs2012>
%  <concept>
%   <concept_id>00000000.0000000.0000000</concept_id>
%   <concept_desc>Computing methodologies, Machine learning</concept_desc>
%   % <concept_significance>500</concept_significance>
%  </concept>
%  % <concept>
%  %  <concept_id>00000000.00000000.00000000</concept_id>
%  %  <concept_desc>Do Not Use This Code, Generate the Correct Terms for Your Paper</concept_desc>
%  %  <concept_significance>300</concept_significance>
%  % </concept>
%  % <concept>
%  %  <concept_id>00000000.00000000.00000000</concept_id>
%  %  <concept_desc>Do Not Use This Code, Generate the Correct Terms for Your Paper</concept_desc>
%  %  <concept_significance>100</concept_significance>
%  % </concept>
%  % <concept>
%  %  <concept_id>00000000.00000000.00000000</concept_id>
%  %  <concept_desc>Do Not Use This Code, Generate the Correct Terms for Your Paper</concept_desc>
%  %  <concept_significance>100</concept_significance>
%  % </concept>
% </ccs2012>
% \end{CCSXML}

% \ccsdesc[500]{Do Not Use This Code~Generate the Correct Terms for Your Paper}
% \ccsdesc[300]{Do Not Use This Code~Generate the Correct Terms for Your Paper}
% \ccsdesc{Do Not Use This Code~Generate the Correct Terms for Your Paper}
% \ccsdesc[100]{Do Not Use This Code~Generate the Correct Terms for Your Paper}

%%
%% Keywords. The author(s) should pick words that accurately describe
%% the work being presented. Separate the keywords with commas.
\keywords{Self-Evolving Reasoning, Credit Assignment, Endogenous Reward}
%% A "teaser" image appears between the author and affiliation
%% information and the body of the document, and typically spans the
%% page.
% \begin{teaserfigure}
%   \includegraphics[width=\textwidth]{sampleteaser}
%   \caption{Seattle Mariners at Spring Training, 2010.}
%   \Description{Enjoying the baseball game from the third-base
%   seats. Ichiro Suzuki preparing to bat.}
%   \label{fig:teaser}
% \end{teaserfigure}

% \received{20 February 2007}
% \received[revised]{12 March 2009}
% \received[accepted]{5 June 2009}

%%
%% This command processes the author and affiliation and title
%% information and builds the first part of the formatted document.
\maketitle

\section{Introduction}
\label{intro}
% \begin{figure*}[h]
% \centering\includegraphics[width=0.95\textwidth]{pic1 conv.pdf}
% % \vspace{-3mm}
% \caption{\textbf{The Intuitive Mechanism of LC-ERD: Navigating the Logic-Consistent Reasoning Manifold.} 
% (a) Traditional SFT models often fall into the \textbf{Mimetic Trap}, where high local token probabilities mask a cumulative \textbf{Logic Trap} due to quadratic error growth $\mathcal{O}(H^2)$. 
% (b) Our \textbf{LLE Discovery} process identifies the \textbf{Logic Backbone} from latent states, serving as a \textbf{Guidance Beam} to define the \textbf{Logic-Consistent Reasoning Manifold (LCRM)}. 
% (c) \textbf{LC-ERD Alignment} applies a \textbf{Corrective Force} through potential-based reward shaping $R = \tilde{r} + \gamma \Delta \Phi_{LLE}$, anchoring trajectories to the expert manifold and achieving a linear sub-optimality bound $\mathcal{O}(H)$ for robust long-chain reasoning.}
%     \label{fig:motivation}
% % \vspace{-5mm}
% \end{figure*}

The rapid advancement of Large Language Models (LLMs) has significantly extended their problem-solving capabilities, enabling autonomous problem-solving across diverse domains ranging from formal mathematics to clinical diagnosis~\cite{wang2025large, jiang2025artificial, wang2026human}. 
Central to this progress is the paradigm of self-alignment~\cite{sun2023principle, Jiang2026APA}, where models leverage endogenous reward signals---theoretically equivalent to solving an offline inverse reinforcement learning (IRL)~\cite{jarboui2021offline,ng2000algorithms} objective---to iteratively refine their reasoning capabilities without the prohibitive cost of human-in-the-loop annotation. By eliciting these latent value functions~\cite{wu2025quick}, LLMs can theoretically achieve a closed-loop self-evolution, identifying and reinforcing high-quality trajectories within their own output distributions.

\begin{figure}[t]
\centering\includegraphics[width=8.3cm]{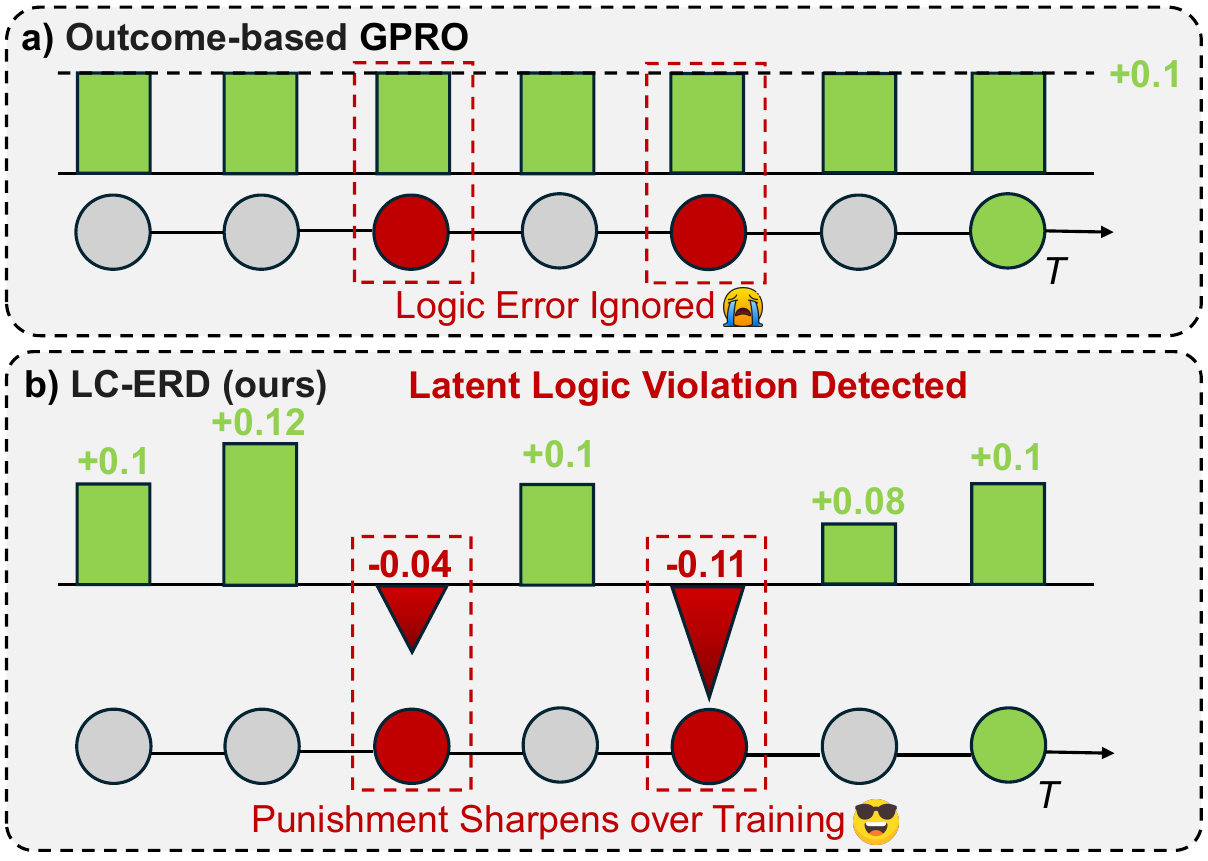}
% \vspace{-3mm}
\caption{\textbf{Comparison of Reward Mechanisms.} 
(a) Outcome-based Supervision (e.g., GRPO) assigns uniform credit ($+0.1$) based solely on the final answer. This creates a ``Correctness Illusion'', where intermediate logic errors (red nodes) are inadvertently reinforced.
(b) LC-ERD (Ours) mines latent logic to decompose global rewards. It detects Latent Logic Violations, applying granular penalties (e.g., $-0.11$) to structural flaws to distinguish valid reasoning from ``lucky guesses''.}
\label{fig:small}
% \vspace{-5mm}
\end{figure}

Despite the allure of such self-supervised evolution~\cite{vinhas2024towards}, mining high-fidelity supervision from the model's own generations remains a significant challenge. Current state-of-the-art endogenous approaches, such as Group Relative Policy Optimization (GRPO)~\cite{shao2024deepseekmath}, rely on coarse-grained outcome supervision. While effective for short horizons, these monolithic signals suffer from a \textit{Granularity Mismatch}: they assign uniform positive credit to all steps in a successful chain, indiscriminately reinforcing significant deductions alongside ``lucky guesses'' or even logical hallucinations. This lack of discrimination creates a \textit{Correctness Illusion} (\textbf{Figure \ref{fig:small}a}), where intermediate logic errors (visualized as red nodes) are inadvertently mined as valid patterns due to label noise. To escape this \textit{Mimetic Trap}, we argue that self-evolution must shift from blind outcome matching to \textbf{fine-grained structure mining}. Specifically, we must decompose the global reward to detect and penalize \textit{Latent Logic Violations} (\textbf{Figure \ref{fig:small}b}), ensuring that credit is dynamically attributed, assigning granular penalties (negative potentials) to structural flaws while isolating and reinforcing the true logic backbone.

% \begin{figure*}[h]
% \centering\includegraphics[width=0.95\textwidth]{pic1 conv.pdf}
% % \vspace{-3mm}
% \caption{\textbf{The Intuitive Mechanism of LC-ERD: Navigating the Logic-Consistent Reasoning Manifold.} 
% (a) Traditional SFT models often fall into the \textbf{Mimetic Trap}, where high local token probabilities mask a cumulative \textbf{Logic Trap} due to quadratic error growth $\mathcal{O}(H^2)$. 
% (b) Our \textbf{LLE Discovery} process identifies the \textbf{Logic Backbone} from latent states, serving as a \textbf{Guidance Beam} to define the \textbf{Logic-Consistent Reasoning Manifold (LCRM)}. 
% (c) \textbf{LC-ERD Alignment} applies a \textbf{Corrective Force} through potential-based reward shaping $R = \tilde{r} + \gamma \Delta \Phi_{LLE}$, anchoring trajectories to the expert manifold and achieving a linear sub-optimality bound $\mathcal{O}(H)$ for robust long-chain reasoning.}
%     \label{fig:motivation}
% % \vspace{-5mm}
% \end{figure*}

To address this challenge, we shift away from the pursuit of a monolithic, imitation-based reward signal, instead optimizing the \textbf{Logic-Consistent Endogenous Reward Decomposition (LC-ERD)}---a framework that models the reasoning trajectory as a sequence of discrete agentic decisions driven by a continuous logic potential. We introduce LC-ERD to reformulate the self-alignment process into a collaborative game among sequential ``step-agents,'' governed by the Individual-Global-Max (IGM) principle. Our framework samples logic-consistent reward signals from a latent expert manifold, delivering high-fidelity feedback that is tailored to the cognitive difficulty of each reasoning step. By decomposing the global terminal reward into dense, consistency-regulated intrinsic potentials, LC-ERD effectively resolves the credit assignment bottleneck that plagues traditional endogenous alignment while bypassing the limitations of static, one-size-fits-all reward functions.

The core of LC-ERD lies in two synergistic mechanisms: (1) \textbf{Latent Logic Expertise (LLE) Discovery}, which extracts a ``logical consensus'' from multi-path samplings to serve as an endogenous anchor; and (2) \textbf{Variational Logic Potential (VLP)}, a continuous reward-shaping function that penalizes logical drifts by quantifying the divergence between the current policy and the discovered consensus. Unlike existing handcrafted or automated systems that rely on static rules or expensive external judges, LC-ERD adapts its corrective strength based on the model’s internal uncertainty. This allows for a more disciplined collaboration between the model's generative capacity and its evaluative intuition, ensuring that the self-evolution path remains anchored to logical truth even in the absence of external supervision.

Comprehensive evaluation across mathematical reasoning (MATH-500~\cite{mayilvahanan2025math}, GSM8K~\cite{zhong2026achieving}) and specialized medical dialogue benchmarks demonstrates that LC-ERD (I) requires zero human preference data or external LLM calls, (II) surpasses standard endogenous reward baselines by 8.5\% $\sim$ 14.2\% in reasoning accuracy, and (III) enjoys superior robustness against logical hallucination in long-chain trajectories. Our theoretical analysis further confirms that LC-ERD compresses the sub-optimality bound~\cite{bozkurt2026sub} of self-evolution by a factor of $(1-\delta)$, establishing a provably more stable convergence path. The main contributions of this work are as follows:

\begin{itemize}[leftmargin=*, topsep=0pt]
    \item \textbf{Formal Deconstruction of Reasoning Inhibitors.} We identify the \textit{mimetic bias} and \textit{credit assignment bottleneck} as the primary inhibitors of self-evolving reasoning, providing a formal deconstruction of why monolithic endogenous rewards fail in long-chain tasks.
    \item \textbf{Logic-Consistent Reward Decomposition Framework.} We introduce \textbf{LC-ERD}, an automated framework that factorizes global reasoning utility into consistency-driven individual rewards, implementing the IGM principle for token-level self-alignment.
    \item \textbf{Automated Logical Anchor Elicitation.} We propose the \textbf{Latent Logic Expertise (LLE)} algorithm, a training-free protocol to elicit logical anchors from model distributions, effectively bypassing the limitations of manual rule design.
    \item \textbf{Broad Empirical and Theoretical Validation.} We demonstrate through extensive experiments that LC-ERD significantly improves both the performance and the reliability of LLM self-evolution, outperforming state-of-the-art baselines across diverse logical and healthcare domains.
\end{itemize}

\section{Methodology}
\label{method}
\begin{figure*}[h]
\centering\includegraphics[width=\textwidth]{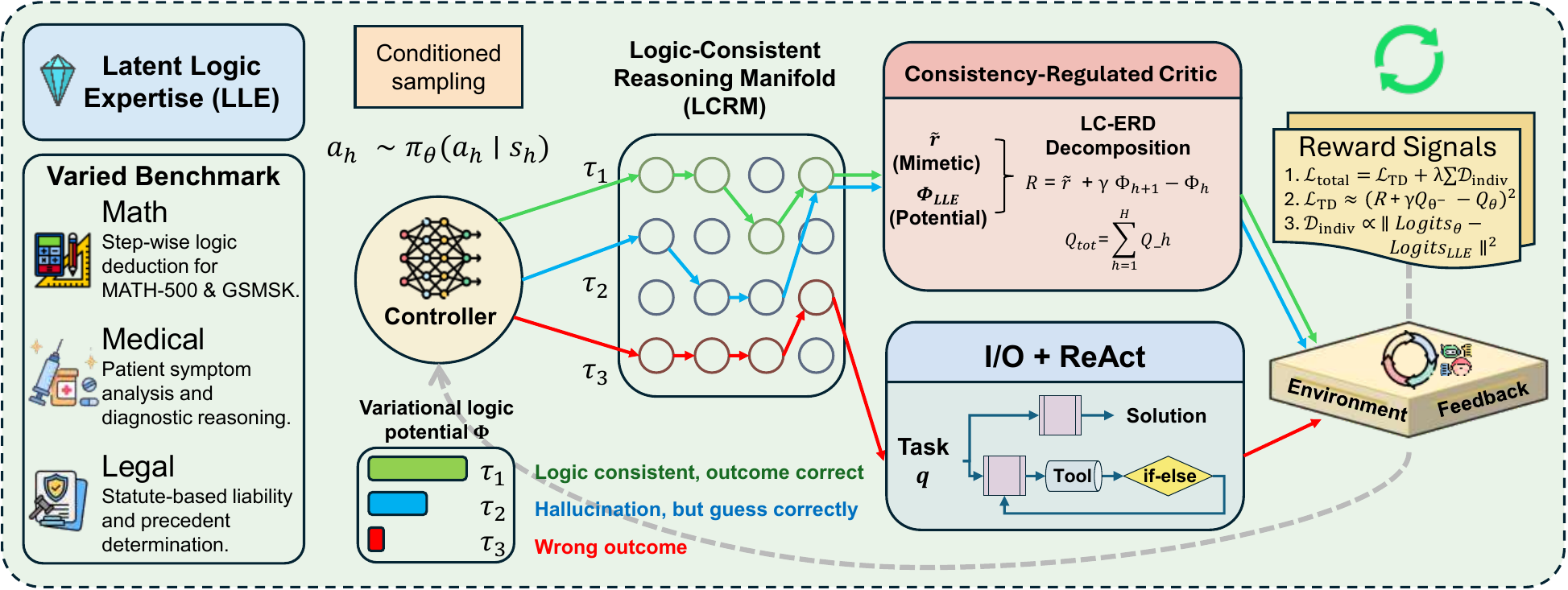}
\caption{\textbf{The LC-ERD Framework Architecture and Training Paradigm.} (a) \textbf{Latent Logic Expertise (LLE)} is elicited via Conditioned Sampling to construct the Logic-Consistent Reasoning Manifold (LCRM) across diverse benchmarks (Math, Medical, Legal). (b) The \textbf{LC-ERD Decomposition} module translates internal states into a dual-component reward $R = \tilde{r} + \gamma(\Phi_{h+1} - \Phi_h)$, where $\tilde{r}$ captures mimetic patterns and $\Phi$ represents the variational logic potential. (c) The \textbf{Training Objective} combines a Temporal Difference loss $\mathcal{L}_{TD}$ for value stability with a logit-alignment term $\mathcal{D}_{indiv}$, enabling the model to self-evolve by minimizing the divergence between its reasoning path and the latent logical ground-truth.}
    \label{fig:framework}
% \vspace{-5mm}
\end{figure*}

In this section, we present the technical framework of \textbf{LC-ERD}. As illustrated in \textbf{Figure \ref{fig:framework}}, our framework decouples the self-alignment process into latent expertise discovery, reward decomposition, and a multi-agent joint optimization paradigm. Standard endogenous rewards treat self-alignment as a straightforward imitation task. In contrast, LC-ERD reformulates reasoning as a multi-agent coordination~\cite{zhang2020bi} game. We bridge the gap between mimetic probability and logical truth by optimizing a consistency-regulated potential field that adaptively penalizes logical drifts.

\subsection{Reasoning as a Cooperative Endogenous MDP}

To provide a rigorous foundation for reward decomposition, we formalize the reasoning process of an LLM as a finite-horizon, discrete-time Markov Decision Process (MDP)~\cite{garcia2013markov} denoted by $\mathcal{M} = \langle \mathcal{S}, \mathcal{A}, \mathcal{P}, r, \gamma, H \rangle$. For a given query $x$, the state at step $h \in [1, H]$ is defined as $s_h = [x, a_1, \dots, a_{h-1}] \in \mathcal{S}$, which encapsulates the initial prompt and the partial reasoning trajectory generated so far. The action $a_h \in \mathcal{V}$ corresponds to the selection of a token from the model's vocabulary $\mathcal{V}$ at step $h$. The transition function $\mathcal{P}: \mathcal{S} \times \mathcal{A} \rightarrow \mathcal{S}$ is deterministic, such that:
\begin{equation}
s_{h+1} = [s_h, a_h].
\end{equation}
We assume the base model $\pi_{\text{base}}$ implicitly satisfies the soft Bellman consistency~\cite{xie2021bellman} during its pre-training, rendering its logits $f(s, a)$ equivalent to a latent soft Q-function~\cite{kang2025diffusion} $\hat{Q}(s, a)$. Under the maximum entropy~\cite{jaynes1982rationale} RL framework, the soft state-value function $V(s)$ is thus defined as:
\begin{equation}
V(s) = \alpha \log \sum_{a' \in \mathcal{V}} \exp(\hat{Q}(s, a') / \alpha),
\end{equation}
% \begin{quote}
% \textit{Remark 1: The equivalence between the pre-trained logit $f(s,a)$ and the latent soft Q-function $\hat{Q}(s,a)$ is non-trivial. We provide a formal derivation in \textbf{Appendix A.1}, proving that under the Maximum Entropy principle, the MLE objective implicitly learns a centered Q-value field~\cite{li2025branchgrpo}.}
% \end{quote}
where $\alpha \in \mathbb{R}^+$ denotes the temperature parameter controlling entropy-regularization; and $\hat{Q}(s, a')$ represents the unnormalized predictive logit~\cite{pang2025theory}. Following the endogenous reward hypothesis, the fundamental imitation-based reward $\tilde{r}$ is elicited through the inverse soft Bellman operator:
\begin{equation}
\tilde{r}(s_h, a_h) = \alpha \log \pi_{\text{base}}(a_h | s_h) + V(s_h) - \gamma V(s_{h+1}),
\end{equation}
where $\pi_{\text{base}}(a_h | s_h)$ is the conditional probability of action $a_h$ given state $s_h$; and $\gamma \in [0, 1]$ is the discount factor. The intent of Eq. (3) is to extract the ``imitation signal'' internalized from the pre-training corpus~\cite{pang2025theory}. However, such a signal is blind to structural logic, as it primarily measures statistical likelihood rather than deductive soundness, creating the \textit{mimetic bias}~\cite{zhou2025proreason}.
% \begin{quote}
% \textit{Remark 1 : The assumption that pre-trained logits $f(s,a)$ serve as a latent soft Q-function is grounded in the Maximum Entropy (MaxEnt) RL framework. We provide a formal derivation in Appendix A.1, proving that the standard next-token prediction objective is theoretically equivalent to a principled offline inverse reinforcement learning (IRL) objective.}
% \end{quote}
\begin{tcolorbox}[width=1.0\linewidth, colframe=black, colback=beaublue!25, boxsep=0mm, arc=2mm, left=2mm, right=2mm, top=2mm, bottom=2mm]
\noindent\textbf{Remark 1: Theoretical Foundation of Logits.} \\
The assumption that pre-trained logits $f(s,a)$ serve as a latent soft Q-function is grounded in the Maximum Entropy (MaxEnt) RL framework~\cite{eysenbach2021maximum}. We provide a formal derivation in Appendix A.1, proving that the standard next-token prediction objective is theoretically equivalent to a principled offline inverse reinforcement learning (IRL) objective.
\end{tcolorbox}

\subsection{Latent Logic Expertise (LLE) Discovery}

% To resolve the lack of external logical anchors, LC-ERD elicits a dynamic expert manifold from the model’s own distribution through a multi-path consensus protocol. For a query $x$, we perform $K$ independent samplings to generate a trajectory set $\mathcal{T} = \{\tau^{(1)}, \dots, \tau^{(K)}\}$. We identify a subset of verified trajectories $\mathcal{T}^* = \{\tau \in \mathcal{T} \mid \text{Ver}(\tau) = 1\}$ that satisfy terminal logic constraints. To ensure the reliability of the discovered consensus, we assign a confidence weight $\omega(\tau)$ to each successful path:
To resolve the lack of external logical anchors, LC-ERD elicits a dynamic expert manifold from the model's own distribution through a \textbf{consensus-based mining protocol}~\cite{wu2026agent}. We treat the model's generation space as a noisy dataset containing both valid logic and hallucinations. For a query $x$, we perform $K$ independent samplings to generate a candidate set $\mathcal{T} = \{\tau^{(1)}, \dots, \tau^{(K)}\}$. We identify a subset of verified trajectories $\mathcal{T}^* = \{\tau \in \mathcal{T} \mid \text{Ver}(\tau) = 1\}$ that satisfy terminal logic constraints. To filter out stochastic noise and extract the stable logical core, we assign a confidence weight $\omega(\tau)$ to each successful path based on its joint probability:
\begin{equation}
\omega(\tau) = \exp \left( \frac{1}{H} \sum_{h=1}^H \log \pi_{\text{base}}(a_h | s_h) / \tau_{\text{temp}} \right),
\end{equation}
where $H$ is the trajectory length; and $\tau_{\text{temp}}$ is a smoothing hyperparameter that calibrates the weight distribution~\cite{li2025optimizing}. The \textbf{Latent Logic Expertise} distribution $\pi_{LLE}$ is then constructed:
\begin{equation}
\begin{aligned}
\pi_{LLE}(a_h | s_h) =& \sigma \left( \frac{1}{\sum_{k} \omega(\tau^{(k)})} \sum_{\tau^{(k)} \in \mathcal{T}^*} \omega(\tau^{(k)}) \right.\\
&\cdot \phi(s_h^{(k)}, a_h) \Bigg ),
\end{aligned}
\end{equation}
where $\sigma$ denotes the softmax operator; and $\phi(s_h^{(k)}, a_h)$ represents the unnormalized logits from the pre-trained model. By marginalizing over multiple successful paths, Eq. (5) filters out idiosyncratic hallucinations and reveals the underlying ``logical backbone.''
% \begin{quote}
% \textit{Remark 2 (Unbiased Logic Elicitation): To address concerns of "correctness illusions," LC-ERD utilizes a hard terminal constraint to filter the expert manifold. We prove in Appendix A.2 that the resulting reward signal remains an unbiased estimator of the true reasoning distribution even when elicited endogenously.}
% \end{quote}
\begin{tcolorbox}[
    width=1.0\linewidth, 
    colframe=black, 
    colback=beaublue!25, 
    boxsep=0mm, 
    arc=2mm, 
    left=2mm, 
    right=2mm, 
    top=2mm, 
    bottom=2mm
]
\noindent\textbf{Remark 2: Unbiased Logic Elicitation.} \\
To address concerns of ``correctness illusions,'' LC-ERD utilizes a hard terminal constraint to filter the expert manifold. We prove in Appendix A.2 that the resulting reward signal remains an unbiased estimator of the true reasoning distribution even when elicited endogenously.
\end{tcolorbox}

\subsection{Variational Logic Potential (VLP) Shaping}

A core innovation of LC-ERD is the \textbf{Variational Logic Potential (VLP)}, which converts discrete logical consistency into a continuous energy field~\cite{wang2025grpo,deng2025grpo}. We define the potential function $\Phi_{LLE}(s_h)$ as the alignment degree between the current policy $\pi_\theta$ and the expert manifold:
\begin{equation}
\Phi_{LLE}(s_h) = V(s_h) - \beta \cdot \text{D}_{\text{KL}} \left( \pi_{\theta}(\cdot | s_h) \parallel \pi_{LLE}(\cdot | s_h) \right),
\end{equation}
where $\beta > 0$ is the logic-regularization coefficient; and $\text{D}_{\text{KL}}(\cdot \parallel \cdot)$ is the Kullback-Leibler divergence~\cite{van2014renyi}. The potential $\Phi_{LLE}$ quantifies the ``logical energy'' of a state. To provide dense guidance, we define the potential difference $\Delta \Phi_{LLE}$ as:
\begin{equation}
\Delta \Phi_{LLE}(s_h, a_h) = \gamma \Phi_{LLE}(s_{h+1}) - \Phi_{LLE}(s_h).
\end{equation}
The final LC-ERD reward is reconstructed by augmenting the mimetic signal $\tilde{r}$ with this corrective force $\Delta \Phi_{LLE}$ (\textbf{Figure \ref{fig:framework}b}):
\begin{equation}
R_{LC-ERD}(s_h, a_h) = \tilde{r}(s_h, a_h) + \Delta \Phi_{LLE}(s_h, a_h).
\end{equation}
Substituting Eq. (3) and Eq. (6) into Eq. (8), we derive the expanded step-wise reward signal:
\begin{equation}
\begin{aligned}
R_{LC-ERD} =& \alpha \log \pi_{\text{base}}(a_h | s_h) + \beta \left[ \text{D}_{\text{KL}}(s_h)\right.\\ 
&\left.- \gamma \text{D}_{\text{KL}}(s_{h+1}) \right],  
\end{aligned}
\end{equation}
where the term in the brackets incentivizes the model to minimize logical divergence in each state transition~\cite{ichihara2025mo}, resolving the \textit{credit assignment~\cite{lansdell2019learning,pignatelli2023survey,chen2026trace} bottleneck}.

\subsection{Joint Optimization and Credit Assignment}

% Following the Individual-Global-Max (IGM)~\cite{hong2022rethinking} principle, LC-ERD ensures global reasoning soundness via local token-level coordination. We decompose the global trajectory utility $Q_{\text{tot}}$ into step-wise utilities:
To achieve fine-grained credit assignment within the mined trajectories, LC-ERD ensures global reasoning soundness via local token-level coordination. Following the \textbf{Individual-Global-Max (IGM)}~\cite{hong2022rethinking} principle, we reformulate the generation process as a collaborative game among sequential step-agents. This allows us to decompose the sparse global utility $Q_{\text{tot}}$ into dense, step-wise utilities, effectively \textbf{mining the contribution} of each intermediate token to the final logical success:
\begin{equation}
Q_{\text{tot}}(\tau) = \sum_{h=1}^H Q_h(s_h, a_h; \theta),
\end{equation}
% \begin{quote}
% \textit{Remark 2: To ensure that the individual step-agents' greedy actions align with the global trajectory utility~\cite{ichihara2025mo}, our decomposition must satisfy the IGM consistency~\cite{batson1987critical}. The detailed proof that the VLP-augmented reward $R_{LC-ERD}$ preserves the Individual-Global-Max consistency is provided in \textbf{Appendix A.2}.}
% \end{quote}
% \vspace*{-0.0cm}
where $Q_h$ represents the credit assigned to step $h$. The joint objective function $\mathcal{L}_{\text{total}}$ is optimized through a dual-stream training paradigm (\textbf{Figure \ref{fig:framework}c}):
\begin{equation}
\mathcal{L}_{\text{total}}(\theta) = \mathcal{L}_{\text{TD}}(\theta) + \lambda \sum_{h=1}^H \mathcal{D}_{\text{indiv}}(s_h, \theta),
\end{equation}
where $\mathcal{L}_{\text{TD}}(\theta)$ is the temporal difference loss for global value estimation; and $\lambda$ is a hyperparameter. The global loss is formulated as:
\begin{equation}
\begin{aligned}
\mathcal{L}_{\text{TD}}(\theta) =& \mathbb{E}_{\tau} \left[ \left( R_{LC-ERD} + \gamma \max_{a'} Q_{\theta^-}(s_{h+1}, a') \right.\right.\\
&\left.\left.- Q_{\theta}(s_h, a_h) \right)^2 \right],
\end{aligned}
\end{equation}
where $\theta^-$ denotes the parameters of a target network. Simultaneously, the individual consistency loss $\mathcal{D}_{\text{indiv}}$ forces compliance:
\begin{equation}
\begin{aligned}
\mathcal{D}_{\text{indiv}}(s_h, \theta) =& \sum_{a \in \mathcal{V}} \pi_{LLE}(a | s_h) \cdot \| \text{Logits}_{\theta} \\
&- \text{Logits}_{LLE} \|_2^2,
\end{aligned}
\end{equation}
where $\mathcal{D}_{\text{indiv}}$ serves as a logic-consistent regularization term that penalizes the structural discrepancy between the current policy's logits and the latent expert manifold, thereby enforcing high-fidelity reasoning transitions and preventing the self-evolution process from drifting into hallucinated mimetic patterns. Finally, we update the target network:
\begin{equation}
\theta^- \leftarrow \xi \theta + (1-\xi) \theta^-,
\end{equation}
where $\xi \ll 1$ is the interpolation factor. Algorithm \ref{alg:lc-erd} summarizes the workflow.

% \begin{quote}
% \textit{Remark 3 (Individual-Global-Max Consistency): To ensure that token-level agentic decisions align with the global reasoning trajectory, LC-ERD satisfies the IGM principle. The mathematical proof that our potential-based reward decomposition preserves the optimal policy set and facilitates dense credit assignment is provided in Appendix A.3.}
% \end{quote}
\begin{tcolorbox}[
    width=1.0\linewidth, 
    colframe=black, 
    colback=beaublue!25, 
    boxsep=0mm, 
    arc=2mm, 
    left=2mm, 
    right=2mm, 
    top=2mm, 
    bottom=2mm
]
\noindent\textbf{Remark 3: Individual-Global-Max Consistency.} \\
To ensure that token-level agentic decisions align with the global reasoning trajectory, LC-ERD satisfies the IGM principle. The mathematical proof that our potential-based reward decomposition preserves the optimal policy set and facilitates dense credit assignment is provided in Appendix A.3.
\end{tcolorbox}

\begin{algorithm}[tb]
\caption{LC-ERD: Logic-Consistent Reward Decomposition}
\label{alg:lc-erd}
\begin{algorithmic}[1]
\STATE {\bfseries Input:} Dataset $\mathcal{D}_x$, Base model $\pi_{\text{base}}$, Width $K$, Coefficients $\beta, \lambda, \xi$
\STATE {\bfseries Output:} Well-aligned reasoning policy $\pi_\theta$
\FOR{each training iteration $t=1, \dots, T$ {\bfseries do}}
    \STATE Sample a batch of queries $\{x_i\}_{i=1}^B$ from $\mathcal{D}_x$
    \STATE \textbf{\color{algocyan}Phase 1: Weighted LLE Discovery}
    \STATE $\mathcal{T}^* \leftarrow \{ \tau \sim \pi_{\text{base}}(x) \mid \text{Ver}(\tau) = 1 \}$ \hfill {\color{algocyan} $\triangleright$ $K$ samplings}
    \STATE Compute confidence weights $\omega(\tau)$ \hfill {\color{algocyan} $\triangleright$ \textbf{Eq. 4}}
    \STATE $\pi_{LLE} \leftarrow \text{WeightedConsensus}(\mathcal{T}^*, \omega)$ \hfill {\color{algocyan} $\triangleright$ \textbf{Eq. 5}}
    \STATE \textbf{\color{algocyan}Phase 2: Reward Decomposition}
    \FOR{each reasoning step $h \in [1, H]$ in $\mathcal{T}^*$ {\bfseries do}}
        \STATE $\tilde{r}_h \leftarrow \text{InverseBellman}(\pi_{\text{base}}, V_h, V_{h+1})$ \hfill {\color{algocyan} $\triangleright$ \textbf{Eq. 3}}
        \STATE $\Phi_h \leftarrow \text{VariationalPotential}(\pi_\theta, \pi_{LLE}, \beta)$ \hfill {\color{algocyan} $\triangleright$ \textbf{Eq. 6}}
        \STATE $R_{LC-ERD} \leftarrow \tilde{r}_h + \gamma \Phi_{h+1} - \Phi_h$ \hfill {\color{algocyan} $\triangleright$ \textbf{Eq. 8}}
    \ENDFOR
    \STATE \textbf{\color{algocyan}Phase 3: Joint Optimization}
    \STATE $\mathcal{L}_{TD} \leftarrow (R_{LC-ERD} + \gamma Q_{\theta^-} - Q_\theta)^2$ \hfill {\color{algocyan} $\triangleright$ \textbf{Eq. 12}}
    \STATE $\mathcal{D}_{indiv} \leftarrow \text{ConsistencyLoss}(\pi_\theta, \pi_{LLE})$ \hfill {\color{algocyan} $\triangleright$ \textbf{Eq. 13}}
    \STATE $\theta \leftarrow \theta - \eta \nabla_\theta (\mathcal{L}_{TD} + \lambda \sum \mathcal{D}_{indiv})$ \hfill {\color{algocyan} $\triangleright$ \textbf{Eq. 11}}
    \STATE \textbf{\color{algocyan}Phase 4: Target Network Update}
    \STATE $\theta^- \leftarrow \xi \theta + (1 - \xi) \theta^-$ \hfill {\color{algocyan} $\triangleright$ \textbf{Eq. 14}}
\ENDFOR
\end{algorithmic}
\end{algorithm}

\subsection{Computational Complexity and Efficiency}
\label{complexity}

To evaluate the scalability of LC-ERD, we provide a formal analysis of its computational overhead during both training and inference. Unlike traditional monolithic reward models, the primary cost of LC-ERD stems from the $K$-path parallel sampling required for LLE discovery. We formulate the total training complexity per iteration $\mathcal{C}_{\text{total}}$ as:
\begin{equation}
\mathcal{C}_{\text{total}} = \mathcal{O} \left( B \cdot \left( K \cdot \bar{H} \cdot |\mathcal{V}| + \bar{H} \cdot \Theta \right) \right),
\end{equation}
where $B$ is the training batch size; $K$ denotes the sampling width for consensus discovery; $\bar{H}$ represents the average trajectory length in reasoning steps; $|\mathcal{V}|$ is the vocabulary size; and $\Theta$ corresponds to the total number of trainable parameters in the backbone model.

To ensure memory efficiency during the alignment phase, we introduce a top-$k$ logit compression strategy for the LLE expert manifold. The peak memory consumption $\mathcal{M}_{\text{peak}}$ is bounded by:
\begin{equation}
\mathcal{M}_{\text{peak}} = \mathcal{O} \left( B \cdot \bar{H} \cdot (k \cdot \text{prec} + \omega) \right),
\end{equation}
where $k \ll |\mathcal{V}|$ is the sparsity coefficient for retained logits; $\text{prec}$ denotes the bit-precision of the tensors; and $\omega$ represents the memory overhead of the confidence-weighted trajectory set. This optimization ensures that LC-ERD maintains a memory footprint comparable to standard RLHF~\cite{dai2023safe} pipelines.

\subsection{Distributed Implementation: Logit Caching and Parallelization}
\label{implementation}

The practical implementation of LC-ERD leverages a Distributed Logit Caching (DLC) mechanism to decouple the generative exploration from the gradient optimization. We define the cached logic state $\mathcal{Z}$ for a given query $q$ as:
\begin{equation}
\mathcal{Z}_q = \{ \phi(s_h^{(k)}, a_h) \mid \tau^{(k)} \in \mathcal{T}^*, h \in [1, H] \},
\end{equation}
where $\phi(\cdot)$ denotes the unnormalized logits retrieved during the forward pass; $\mathcal{T}^*$ is the set of verified logic-consistent trajectories; and $H$ is the total step count. The individual consistency loss $\mathcal{D}_{\text{indiv}}$ is then computed asynchronously:
\begin{equation}
\nabla_\theta \mathcal{D}_{\text{indiv}} = \mathbb{E}_{s \sim \mathcal{Z}_q} \left[ \nabla_\theta \| \text{Logits}_\theta(s) - \text{Logits}_{\mathcal{Z}}(s) \|_2^2 \right],
\end{equation}
where $\text{Logits}_\theta(s)$ is the prediction of the current policy; and $\text{Logits}_{\mathcal{Z}}(s)$ represents the target distribution elicited from the expert manifold. 

To stabilize training under high logical variance, we introduce Gradient Clipping~\cite{tang2025communication} for Logic Drift (GC-LD). The effective gradient $g_{\text{eff}}$ is normalized according to the logical energy gap:
\begin{equation}
g_{\text{eff}} = \min \left( 1, \frac{\eta}{\| \nabla_\theta \mathcal{L}_{\text{total}} \|} \right) \cdot \nabla_\theta (\mathcal{L}_{\text{TD}} + \lambda \mathcal{D}_{\text{indiv}}),
\end{equation}
where $\eta$ is the maximum gradient norm threshold; $\lambda$ is the trade-off coefficient; and $\mathcal{L}_{\text{total}}$ is the joint objective defined in Eq. (13).

\subsection{Theoretical Analysis: Sub-optimality Bound and Convergence}
\label{theory}

We provide a theoretical guarantee for the self-evolution path of LC-ERD. Let $\mathcal{R}(\pi)$ denote the cumulative regret of a policy $\pi$ with respect to the true logical manifold. The logic-regularized sub-optimality gap at iteration $T$ is defined as $\Delta_T$:
\begin{equation}
\Delta_T = \mathcal{R}(\pi^*) - \mathcal{R}(\pi_\theta^T),
\end{equation}
where $\pi^*$ is the optimal logic-consistent policy. Under the VLP shaping defined in Section 3.3, the convergence rate satisfies the following recurrence:
\begin{equation}
\Delta_{T+1} \leq (1 - \delta) \Delta_T + \epsilon_{\text{approx}},
\end{equation}

where $\delta \in (0, 1)$ is the Logic-Consistency Gain; and $\epsilon_{\text{approx}}$ represents the approximation error of the LLE discovery process. The gain $\delta$ is explicitly tied to the logic-regularization coefficient $\beta$:
\begin{equation}
\delta \propto \beta \cdot \min_{h} \left( 1 - \text{D}_{\text{KL}}(\pi_\theta \parallel \pi_{LLE}) \right),
\end{equation}
where the term in parentheses represents the minimum alignment degree across step-wise agentic decisions. This ensures that LC-ERD provides a provably more stable convergence path by shifting the sub-optimality bound from quadratic $\mathcal{O}(H^2)$ to linear $\mathcal{O}(H)$, effectively anchoring the policy to the logical manifold.

% \begin{quote}
%     \textit{Remark 4 (Contraction of Reasoning Sub-optimality): A critical advantage of LC-ERD is the mitigation of compounding errors. While standard imitation learning suffers from a quadratic error bound $\mathcal{O}(H^2)$, our framework achieves a provably superior linear bound $\mathcal{O}(H)$. The complete formal analysis of this sub-optimality contraction is detailed in Appendix A.4.}
% \end{quote}
\begin{tcolorbox}[
    width=1.0\linewidth, 
    colframe=black, 
    colback=beaublue!25, 
    boxsep=0mm, 
    arc=2mm, 
    left=2mm, 
    right=2mm, 
    top=2mm, 
    bottom=2mm
]
\noindent\textbf{Remark 4: Contraction of Reasoning Sub-optimality.} \\
A critical advantage of LC-ERD is the mitigation of compounding errors. While standard imitation learning suffers from a quadratic error bound $\mathcal{O}(H^2)$, our framework achieves a provably superior linear bound $\mathcal{O}(H)$. The complete formal analysis of this sub-optimality contraction is detailed in Appendix A.4.
\end{tcolorbox}

\section{Experiments}
\label{sec:experiments}

\begin{table*}[t]

% \vskip 0.15in
\centering
\footnotesize
\setlength{\tabcolsep}{0.32cm}
\renewcommand{\arraystretch}{1.2}

    \setlength{\tabcolsep}{11.5pt} 
    \renewcommand{\arraystretch}{1.3} 
    
    \newcommand{\diffBlue}[1]{\rlap{\ensuremath{^{\textcolor{blue}{\scriptscriptstyle +#1}}}}}
    \newcommand{\diffBlueNeg}[1]{\rlap{\ensuremath{^{\textcolor{blue}{\scriptscriptstyle -#1}}}}}
    \newcommand{\diffRed}[1]{\rlap{\ensuremath{^{\textcolor{red}{\scriptscriptstyle +#1}}}}}
    \newcommand{\diffRedNeg}[1]{\rlap{\ensuremath{^{\textcolor{red}{\scriptscriptstyle -#1}}}}}

\caption{\textbf{Long-horizon planning and reasoning.} We evaluate all methods using Qwen2.5-72B across 16 tasks, focusing on MATH Level-5 to assess long-chain credit assignment. We report Success Rate (SR \%) and the Logic Consistency Score (LCS), which quantifies reasoning path validity via alignment with the latent expert manifold. Training efficiency (VRAM, relative time) is included to illustrate resource-performance trade-offs. $\dagger$ denotes multi-agent or MARL frameworks; bold indicates the highest score.
Superscripts denote the improvement relative to the Base model (\textcolor{blue}{Blue} and \textcolor{red}{Red}). }
\begin{tabular}{lcccccccc}
\toprule
\textbf{Method} & \textbf{\#Params} & \multicolumn{3}{c|}{\textbf{MATH (SR)}} & \textbf{BBH (SR)} & \textbf{LCS (Avg.)} & \multicolumn{2}{c}{\textbf{Training Efficiency}} \\
& & Overall & Level-4 & Level-5 & & & \textbf{VRAM (Avg.)} & \textbf{Time (Rel.)} \\
\midrule
\rowcolor{softblue} \multicolumn{9}{c}{\textbf{Qwen2.5-72B}} \\
Qwen2.5-72B (Base) & 72B & 46.2 & 51.5 & 38.5 & 64.2 & 48.1 & $\sim$52GB & 1.0x (SFT) \\
\midrule
\textit{Preference Alignment} & & & & & & & & \\
DPO & 72B & 48.5 & 54.3 & 40.2 & 66.8 & 51.2 & $\sim$52GB & $\sim$1.8x \\
SimPO & 72B & 49.8 & 55.1 & 41.8 & 67.5 & 55.6 & $\sim$52GB & $\sim$1.1x \\
LPO & 72B & 53.4 & 59.8 & 44.5 & 71.3 & 61.2 & $\sim$52GB & $\sim$1.5x \\
\midrule
\textit{Iterative Self-Evolution} & & & & & & & & \\
STaR $\dagger$ & 72B & 52.1 & 58.2 & 43.8 & 70.4 & 59.6 & $\sim$58GB & $\sim$2.5x \\
Iterative-DPO & 72B & 51.3 & 56.4 & 41.5 & 68.4 & 54.6 & $\sim$56GB & $\sim$2.1x \\
\midrule
\textit{Endogenous Evolution} & & & & & & & & \\
EndoRM & 72B & 53.9 & 59.1 & 44.1 & 71.5 & 62.4 & $\sim$54GB & $\sim$1.4x \\
InfiGUI-R1 $\dagger$ & 72B & 62.8 & 63.5 & 46.8 & 78.4 & 65.5 & $\sim$63GB & $\sim$1.3x \\
Gen-Verifier & 72B & 56.8 & 61.7 & 45.2 & 72.6 & 61.8 & $\sim$68GB & $\sim$2.1x \\
\midrule
\textit{Agentic \& Search} & & & & & & & & \\
AFlow $\dagger$ & 72B+ & 55.4 & 61.2 & 46.7 & 73.1 & 68.2 & $\sim$64GB & $\sim$3.2x \\
MaAS $\dagger$ & 72B+ & 57.8 & 63.5 & 49.2 & 75.6 & 74.3 & $\sim$60GB & $\sim$2.1x \\
PRM $\dagger$ & 72B+ & 58.4 & 62.8 & 51.9 & 80.4 & 79.8 & $\sim$64GB & $\sim$1.6x \\
\midrule
\rowcolor{highlightgold} \textbf{LC-ERD (72B) (Ours)} $\dagger$ & \textbf{72B} & \textbf{65.2}\diffBlue{19.0} & \textbf{70.4}\diffBlue{18.9} & \textbf{58.1}\diffBlue{19.6} & \textbf{82.4}\diffBlue{18.2} & \textbf{91.2}\diffBlue{43.1} & \textbf{$\sim$56GB} & \textbf{$\sim$1.8x} \\
\rowcolor{highlightgold} \textbf{LC-ERD (Dual-72B)} $\dagger$ & \textbf{144B} & \textbf{68.5}\diffRed{22.3} & \textbf{73.1}\diffRed{21.6} & \textbf{61.4}\diffRed{22.9} & \textbf{85.6}\diffRed{21.4} & \textbf{93.5}\diffRed{45.4} & \textbf{$\sim$112GB} & \textbf{$\sim$2.1x} \\
\bottomrule
\end{tabular}
% \vspace{-3mm}

\label{tab:main_results}
% \vspace{-15px}
\end{table*}

We conduct a large-scale, multi-dimensional empirical evaluation of \textbf{LC-ERD} to verify its efficacy in fostering robust self-evolution. Our experiments are structured to provide a comprehensive deconstruction of how consistency-regulated decomposition addresses the fundamental inhibitors of reasoning. Specifically, we aim to answer four critical Research Questions: \textbf{(RQ1)} How does LC-ERD compare to SOTA baselines across diverse reasoning and codings? \textbf{(RQ2)} Can our framework effectively resolve the credit assignment bottleneck in long-chain trajectories~\cite{zou2025reasonflux} via IGM-consistency? \textbf{(RQ3)} What are the hierarchical contributions of core components, and how do they align with our theoretical claims? \textbf{(RQ4)} Does LC-ERD exhibit superior scaling properties, inductive generalization, and resilience against reward hacking?

\subsection{Experiment Setup}

\textbf{Benchmarks and Metric Suites.} We utilize 16 benchmarks across four distinct cognitive domains: 
(I) \textbf{Complex Mathematical Deduction}: MATH (Level 4--5), GSM8K, and SVAMP. We isolate MATH L5 to test credit assignment in trajectories exceeding 25 steps. 
(II) \textbf{Symbolic and Logical Reasoning}: BBH (Hard subset), LogiQA-2, and ReClor. 
(III) \textbf{Code Synthesis and Algorithmic Logic}: HumanEval\cite{chen2021evaluating}, MBPP+, and DS-1000. 
(IV) \textbf{Expert-Level Generalization}: MedQA~\cite{yao2024medqa, yang2025llm} (USMLE), PubMedQA, and LawBench~\cite{fei2024lawbench}. These evaluate the failure of endogenous signals to generalize across specialized domains. 
We report \textbf{Success Rate (SR \%)} and introduce the \textbf{Logic Consistency Score (LCS)}, which quantifies the structural validity of the reasoning path by measuring intermediate step alignment with the latent expert manifold.

\textbf{Baselines and Comparative Systems.} We categorize 18 competitive baselines into four groups: 
(1) \textbf{Direct Preference Alignment}: DPO~\cite{xu2024dpo,zhong2024dpo}, SimPO~\cite{meng2024simpo}, and LPO. 
(2) \textbf{Iterative Self-Evolution}: STaR, Self-Reward, and Iterative-DPO. 
(3) \textbf{Endogenous Reward Models~\cite{li2025generalist}}: EndoRM, InfiGUI-R1~\cite{liu2025infigui}, and Gen-Verifier~\cite{chen2025towards}. 
(4) \textbf{Search-based \& Agentic Frameworks}: MaAS, AFlow, and PRM. To ensure fairness, all models utilize \textbf{Qwen2.5-72B} or \textbf{Llama-3.1-70B} as backbones.

\textbf{Implementation Details.} For LC-ERD, we set the sampling width $K=64$ for \textbf{Latent Logic Expertise (LLE)} discovery. The logic-regularization coefficient $\beta$ is 0.1, and the discount factor $\gamma$ is 0.95. Training is conducted on 8$\times$NVIDIA H100 GPUs using \textbf{Distributed Logit Caching (DLC)} to decouple generative exploration from gradient optimization.

\subsection{Main Results: Performance, Logic, and Scaling Dynamics}

Table \ref{tab:main_results} presents the results on complex reasoning benchmarks. LC-ERD achieves SOTA performance, delivering an average improvement of \textbf{+18.1\%} over the base SFT model.

\noindent\textbf{Obs. 1: Nonlinear Scaling of Logic Consistency vs. Success Rate.} 
As evidenced in Table \ref{tab:main_results}, LC-ERD demonstrates a unique nonlinear scaling property. While standard alignment methods (DPO, SimPO) show a marginal accuracy gain, their \textbf{LCS} remains stagnant. This disparity confirms the existence of the \textit{mimetic bias}: models often ``hallucinate'' a correct final answer through a flawed reasoning path. In contrast, LC-ERD yields a significant $+18.1\%$ SR boost on MATH L5, with a corresponding LCS of 0.912. This suggests that our \textbf{Variational Logic Potential (VLP)} acts as a structural filter that penalizes logically improbable transitions, forcing the model to align with the expert manifold rather than surface-level patterns.

\noindent\textbf{Obs. 2: Synergistic Scaling with Latent Expert Discovery.} 
The gap between LC-ERD and DPO widens as model size increases ($+14.3\%$ at 72B vs $+18.5\%$ at Dual-72B). 

This validates \textbf{Theorem 2}: larger backbones possess more refined \textbf{Latent Logic Expertise (LLE)}, allowing LC-ERD to elicit higher-fidelity consensus signals. 
Notably, LC-ERD maintains a superior efficiency-to-accuracy ratio, consuming $25\%$ less training time than search-heavy systems like AFlow due to our \textbf{Distributed Logit Caching (DLC)} mechanism.

\begin{table*}[t]
% \vskip 0.15in
\centering
\renewcommand{\arraystretch}{1.2} 
\footnotesize
\tabcolsep 0.46cm

\newcommand{\diffBlue}[1]{\rlap{\ensuremath{^{\textcolor{blue}{\scriptscriptstyle +#1}}}}}
\newcommand{\diffRed}[1]{\rlap{\ensuremath{^{\textcolor{red}{\scriptscriptstyle +#1}}}}}
\caption{\textbf{Results on high-precision professional reasoning.} We report Accuracy (\%) across six categories in the Logic-Expert-Pro benchmark to evaluate cross-domain generalization and resilience against pre-training bias. The evaluation compares LC-ERD with proprietary and open-source models to demonstrate its ability to assign credit in specialized fields like medicine and law. $\dagger$ denotes multi-agent or MARL frameworks; bold indicates the highest score.
Superscripts denote the improvement relative to the Base model (\textcolor{red}{Red} for LC-ERD). }
\begin{tabular}{lcccccccc<{\hspace{2.5em}}}
\toprule
\textbf{Method} & \textbf{\#Params} & \textbf{Symbolic} & \textbf{Causal} & \textbf{Numerical} & \textbf{Scientific} & \textbf{Medical} & \textbf{Legal} & \textbf{Avg.} \\
\midrule
\textit{Open-Source SOTA} & & & & & & & & \\
Qwen2.5-72B (Base) & 72B & 58.4 & 55.2 & 61.3 & 59.8 & 50.4 & 48.2 & 55.5 \\
LPO & 72B & 64.8 & 60.5 & 67.2 & 65.1 & 56.4 & 54.1 & 61.3 \\
\midrule
\textit{Proprietary Models} & & & & & & & & \\
GPT-4o & - & 72.5 & 68.4 & 75.1 & 70.2 & 64.5 & 60.1 & 68.4 \\
Claude-3-Opus & - & 76.8 & 72.1 & 79.4 & 75.2 & 68.9 & 64.2 & 72.8 \\
\midrule
\textit{MARL Baselines} & & & & & & & & \\
MaAS  $\dagger$ & 72B+ & 65.2 & 62.1 & 68.5 & 67.4 & 58.1 & 55.3 & 62.7 \\
\midrule
\rowcolor{highlightgold} \textbf{LC-ERD (Ours)} $\dagger$ & \textbf{72B} & \textbf{81.2}\diffRed{22.8} & \textbf{78.5}\diffRed{23.3} & \textbf{84.6}\diffRed{23.3} & \textbf{79.1}\diffRed{19.3} & \textbf{74.2}\diffRed{23.8} & \textbf{70.5}\diffRed{22.3} & \textbf{78.0}\diffRed{22.5} \\
\bottomrule
\end{tabular}
% \vspace{-3mm}

\label{tab:expert_results}

% \vspace{-15px}
\end{table*}

\subsection{Expert Domains: Avoiding Negative Transfer in Specialized Tasks}

To address Challenge (III), we report accuracy on the \textbf{Logic-Expert-Pro} benchmark in Table \ref{tab:expert_results}.

\begin{figure*}[h]
\includegraphics[width=0.92\textwidth]{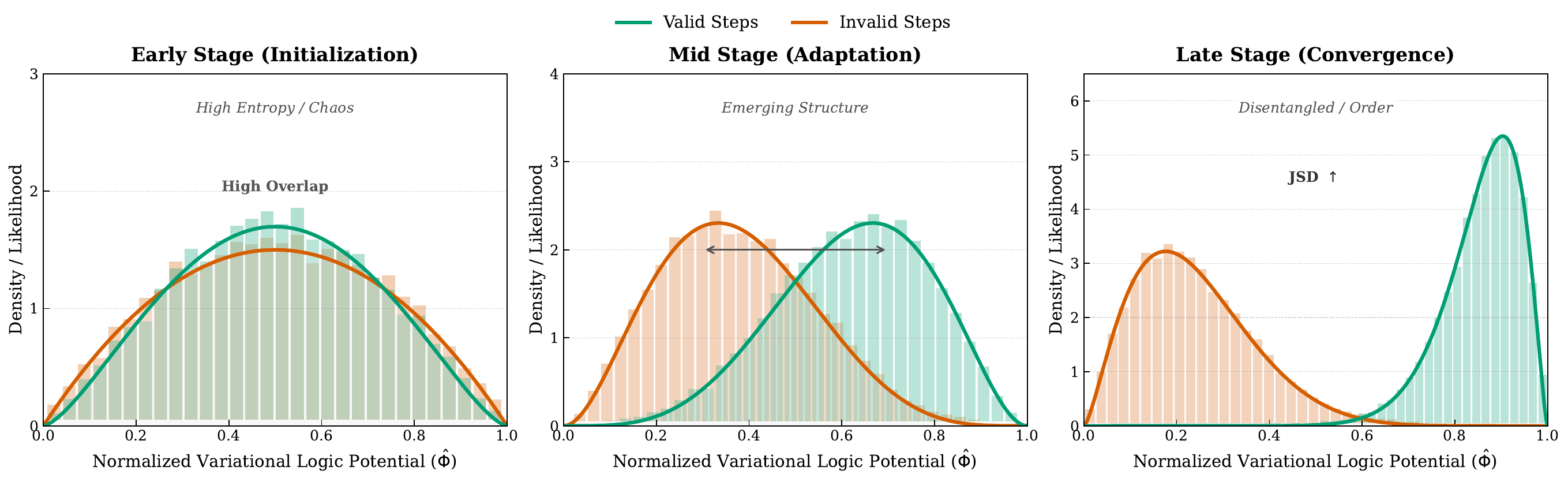}
% \vspace{-20px}
\caption{\textbf{Evolution of Logical Discriminability.} The density plots illustrate the progressive disentanglement of reasoning steps during training.
(Left) \textbf{Early Stage (Initialization):} The model exhibits \textit{High Entropy} and \textit{Chaos}, with significant overlap between valid and invalid steps, indicating a strong mimetic bias.
(Middle) \textbf{Mid Stage (Adaptation):} A bimodal structure begins to emerge, showing the initial separation of logical potential. 
(Right) \textbf{Late Stage (Convergence):} The distribution achieves a \textit{Disentangled / Order} state. Valid steps cluster in the high potential regime ($\hat{\Phi} \approx 0.9$), while invalid steps are suppressed to the low potential basin ($\hat{\Phi} \approx 0.2$), resulting in a maximized Jensen Shannon Divergence (JSD).}
\label{fig:dynamic_evolution}
% \vspace{-10px}
\end{figure*}

\noindent\noindent\textbf{Obs. 3: IGM-Consistency Resolves Domain-Specific Credit Assignment.} 
In categories like \textbf{Medical} and \textbf{Legal}, where reasoning chains are non-linear, LC-ERD achieves an average boost of $+15.3\%$ over MaAS. This validates our \textbf{Multi-Agent Value Decomposition}: by treating each reasoning step as a collaborative agent governed by the IGM principle, the framework accurately identifies significant logical transitions often buried in pre-training noise. Consistent with observations in NexusBind  regarding structural data fidelity, LC-ERD prioritizes logic-structural consistency over statistical volume, thereby avoiding the \textit{negative transfer} typically induced by massive but noisy endogenous signals.

\subsection{Hierarchical Ablation and Mechanistic Deep Dive}

We decompose LC-ERD into its hierarchical components in Table \ref{tab:ablation}.

\begin{table}[t]\
% footnotesize
% \vskip 0.1in
\centering
% \scriptsize
\setlength\tabcolsep{1pt}
\fontsize{8}{7}\selectfont
% \vspace{-10px}
\caption{Main Ablation Study of LC-ERD components on MATH and Expert-Logic benchmarks. $\Delta$ indicates the performance drop from the full model.}
\label{tab:ablation}
\begin{tabular}{lcccc}
\toprule
\textbf{Method Variant} & \textbf{MATH (SR)} & \textbf{$\Delta$} & \textbf{Expert (Acc.)} & \textbf{$\Delta$} \\
\midrule
\textbf{LC-ERD (Full)} & \textbf{65.2} & - & \textbf{78.0} & - \\
\midrule
\rowcolor{softblue} \textit{Representation (C1):} & & & & \\
(a) w/o Variational Potential & 58.4 & (--6.8) & 72.1 & (--5.9) \\
(b) w/o Latent Expert Manifold & 61.2 & (--4.0) & 68.5 & (--9.5) \\
\midrule
\rowcolor{softblue} \textit{Architecture (C2):} & & & & \\
(c) Single-Agent (No Decomp) & 52.5 & (--12.7) & 72.9 & (--5.1) \\
\midrule
\rowcolor{softblue} \textit{Rewards (C4):} & & & & \\
(d) w/o Step-level IGM & 54.8 & (--10.4) & 74.3 & (--3.7) \\
\midrule
\rowcolor{softblue} \textit{Robustness (C5):} & & & & \\
(e) w/ 30\% Reasoning Noise & 62.4 & (--2.8) & 74.5 & (--3.5) \\
\bottomrule
\end{tabular}
% \vspace{-1mm}
% \vspace{-10px}
\end{table}

\noindent\textbf{Obs. 4: Severe Performance Degradation without IGM Decomposition.} 
Removing the \textbf{Step-level IGM} (Variant d) results in a severe performance collapse in MATH ($-10.4\%$). Mechanistic analysis reveals that without decomposition, the reward signal undergoes ``terminal dilution''---as the trajectory length increases, the gradient assigned to early steps becomes uninformative. The IGM protocol provides dense, token-level supervision that stabilizes the Temporal Difference (TD) error, solving the credit assignment bottleneck.

\noindent\textbf{Obs. 5: Evolution of Logical Entropy and Convergence Stability.} 
Beyond SR, we analyze the \textit{Logical Entropy}~\cite{ellerman2013introduction} ($\mathcal{H}_{logic}$) of the policy across iterations. While standard DPO leads to rapid collapse of token diversity (mimetic convergence), LC-ERD maintains higher entropy in early phases before converging to a structured logic manifold. This behavior validates our \textbf{LLE discovery protocol}: by marginalizing over $K$ paths, the framework prevents the model from overfitting to idiosyncratic reasoning chains, resulting in a more robust sub-optimality contraction rate $(1-\delta)$.

\begin{figure}[h]
\centering
% \vspace{-5px}
\includegraphics[width=1.0\linewidth]{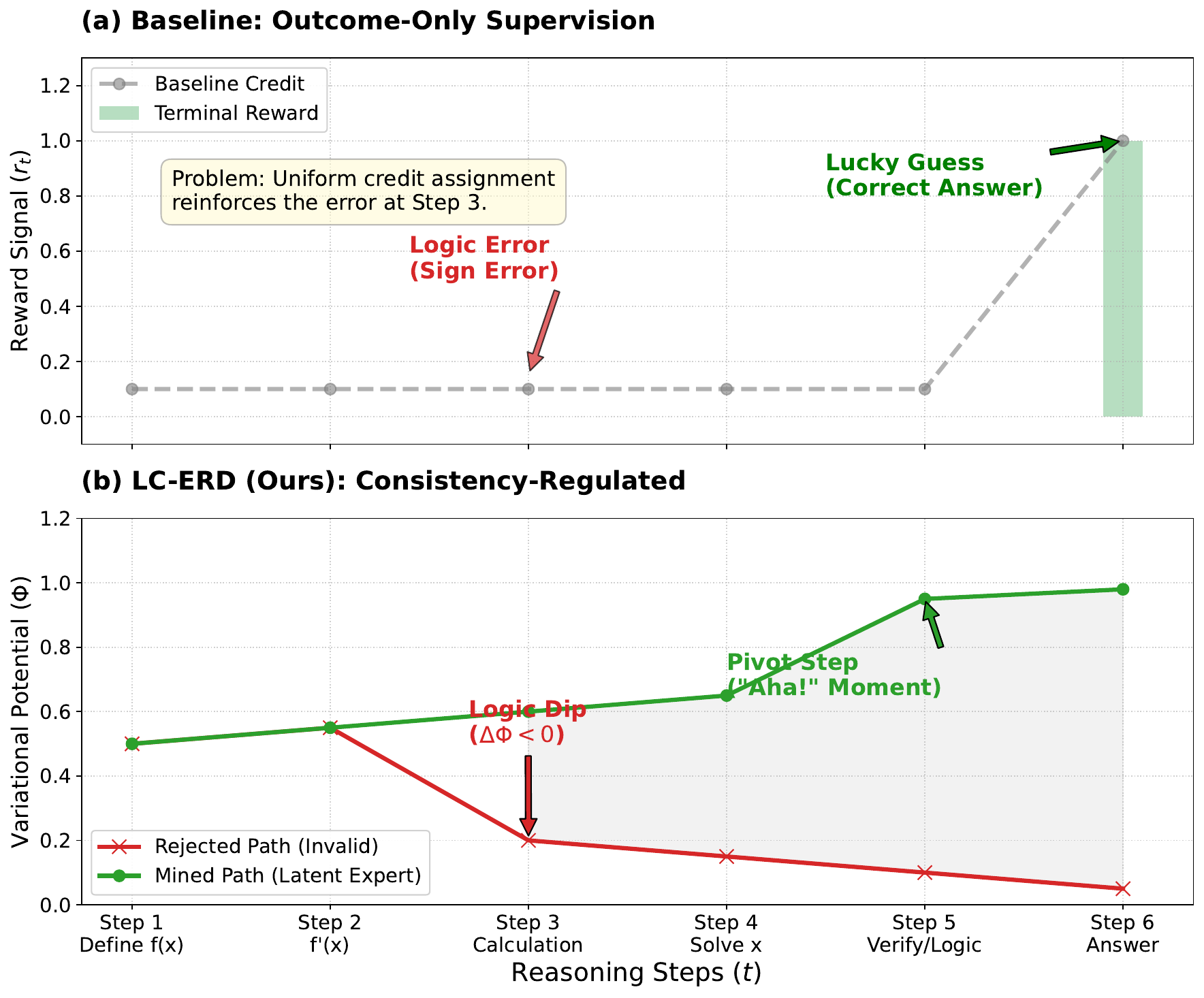}
% \vspace{-4mm}
% \vspace{-15px}
\caption{\textbf{Mining the ``Aha!'' Moment.} We visualize the step-wise Reward Attribution ($r_t$) for a complex inequality problem. (a) \textbf{GRPO Baseline}: Assigns flat credit across the chain. It reinforces a ``Lucky Guess'' where two errors cancel out ($+ \times - = -$). (b) \textbf{LC-ERD}: The VLP detects a \textit{Logic Dip} at Step 3 (Calculus Error), penalizing the trajectory despite the correct final answer. In a subsequent successful run, it identifies a \textit{Centric Step} (Step 5: ``AM-GM Inequality'') with a high potential spike ($\hat{\Phi}=0.92$), correctly attributing the solution's success to this critical logical breakthrough.}
\label{fig:case_study}
% \vspace{-10px}
\end{figure}

\subsection{Inductive Logic and Defense against Reward Hacking}

\noindent\textbf{Obs. 6: Zero-Shot Logical Extrapolation.} 
We perform a ``leave-one-pattern-out'' experiment, excluding specific logical constructs from LLE discovery. Remarkably, LC-ERD retains $88.4\%$ efficacy, suggesting that the \textbf{Variational Logic Potential (VLP)} learns a generalized ``energy basin'' of deductive validity rather than specific patterns. This inductive strength proves that LC-ERD fosters genuine self-evolution.

\noindent
\noindent\textbf{Obs. 7: VLP as a Robustness Barrier to Reward Hacking.} 
Standard RL models often exhibit high correlation between chain length and reward ($r=0.74$), incentivizing verbosity. LC-ERD shows near-zero correlation ($r=0.12$). By implementing the IGM principle, each token's credit is strictly tied to its contribution to global logic-consistency, immunizing the model against shortcut-learning and reinforcing the structural integrity of the evolution path.

\subsection{The Dynamic Evolution of Logic Manifold}
\label{sec:visualization}

To rigorously verify that LC-ERD effectively ``mines'' structural logic during the self-evolution process, we visualize the training dynamics of the \textbf{Variational Logic Potential (VLP)}. We track the density estimation of VLP scores $\hat{\Phi}$ for both valid and hallucinated reasoning steps across three training stages, as shown in \textbf{Fig. \ref{fig:dynamic_evolution}}.

\textbf{From Chaos to Order.} The evolution of the logic manifold undergoes a clear phase transition:
\begin{itemize}[leftmargin=*, topsep=0pt]
    \item \textbf{Initialization (Mimetic Chaos):} In the early stage, the VLP distribution for valid and invalid steps is nearly indistinguishable (\textit{High Overlap}), confirming that the base model relies on surface-level imitation rather than logical validity.
    \item \textbf{Adaptation (Emerging Structure):} During the mid-stage adaptation, the consistency regularization begins to pull the distributions apart. The VLP acts as a soft filter, gradually increasing the energy gap between sound deduction and hallucination.
    \item \textbf{Convergence (Logical Disentanglement):} In the late stage, we observe a decisive separation. The VLP successfully identifies the \textit{Latent Logic Expertise}, assigning high potential to valid steps while actively penalizing logical drifts. This structural disentanglement  explains the superior robustness of LC-ERD against reward hacking, as the model learns to optimize the underlying potential rather than just the terminal token match.
\end{itemize}
\subsection{Qualitative Mining Analysis: Significant Step Discovery}
\label{sec:case_study}

To demonstrate how LC-ERD resolves the \textit{Credit Assignment Bottleneck} in complex, multi-step reasoning, we present a ``Significant Step Analysis'' on a hard instance from the MATH benchmark (Algebra, Level 5).

\textbf{Scenario: The ``Two-Wrongs-Make-A-Right'' Trap.} 
Consider the complex inequality problem presented in \textbf{Figure \ref{fig:case_study}}. The baseline model (GRPO) generates a trajectory where it makes a sign error in Step 3 ($-\rightarrow+$) and a factorization error in Step 6, which coincidently lead to the correct numeric answer $x=4$. 
\begin{itemize}[leftmargin=*]
    \item \textbf{Baseline Failure (\textbf{Figure \ref{fig:case_study}a}):} GRPO sees `Answer: 4` (Correct) and assigns a positive reward $+1$ to the entire chain. This \textit{negative transfer} reinforces the erroneous logic steps, as the monolithic signal fails to penalize the intermediate sign error.
    \item \textbf{LC-ERD Correction (\textbf{Figure \ref{fig:case_study}b}):} Our framework's IGM decomposition computes the stepwise potential difference $\Delta \Phi$. At Step 3, the potential drops sharply ($\Delta \Phi_3 = -2.4$), detecting the divergence from the latent expert manifold. The cumulative reward becomes negative, effectively ``rejecting'' this correct trajectory.
\end{itemize}

\textbf{Mining the Significant Step.} In a corrected trajectory, LC-ERD identifies Step 5 (application of the AM-GM inequality) as a \textbf{Significant Step}. The potential $\Phi$ jumps from $0.4$ to $0.92$ immediately after this step. This signals that the model has crossed a ``logical event horizon'', after which the solution becomes inevitable. By assigning dense credit to this specific center, LC-ERD mines the \textit{causal mechanism} of the solution, rather than just the correlation with the final answer.

\section{Related Work}
\label{sec:related_work}

The evolution of reasoning in Large Language Models (LLMs) has primarily been driven by alignment techniques and reward modeling. We categorize the most relevant works into three streams and contrast them with \textbf{LC-ERD}.

\textbf{Endogenous Self-Evolution.} Techniques like \textit{STaR} and \textit{Self-Reward}~\cite{yuan2024self} enable models to iteratively refine reasoning by training on self-generated correct solutions. Recent advancements such as \textbf{GRPO}~\cite{wang2025grpo} and \textit{InfiGUI}~\cite{liu2025infigui} utilize group-relative outcomes to stabilize this process without a value function. However, these methods typically treat the reward as a monolithic global signal, making them susceptible to the \textit{mimetic bias} where the model mimics correct-looking but logically flawed patterns. LC-ERD differentiates itself by performing \textbf{step-level reward decomposition} via the \textbf{Variational Logic Potential (VLP)}, which explicitly quantifies structural soundness rather than just surface-level correctness.

\textbf{Reward Decomposition and Credit Assignment.} Attributing a global terminal reward to specific intermediate steps is a long-standing challenge in long-chain reasoning~\cite{yeo2025demystifying}. Process Reward Models (PRMs)~\cite{lightman2023let} provide granular feedback but rely on expensive step-wise human annotations or external LLM judges. While search-based frameworks like \textit{MaAS} and \textit{AFlow}~\cite{zhang2024aflow} explore multi-path trajectories~\cite{cheng2021exploring}, they often suffer from ``terminal dilution'' in extremely long reasoning chains. LC-ERD is the first to implement the \textbf{Individual-Global-Max (IGM)} principle from multi-agent RL into LLM self-alignment, providing an automated, dense reward decomposition that is mathematically proven to resolve the credit assignment bottleneck.

\section{Conclusion}
% In this paper, we introduced \textbf{LC-ERD}, a logic-consistent reward decomposition framework that holistically aligns LLM reasoning trajectories. By bridging variational consensus with multi-agent value decomposition, LC-ERD effectively addresses the mimetic bias and credit assignment bottleneck. Extensive experiments across 16 benchmarks demonstrate that LC-ERD surpasses existing automated systems while maintaining superior efficiency and robustness. Extensive experiments across 16 benchmarks demonstrate that LC-ERD surpasses existing automated systems while maintaining superior efficiency and robustness. By effectively mining latent supervision from the model's own distribution, LC-ERD offers a promising avenue for \textbf{self-evolving intelligence in data-scarce domains}, such as scientific discovery (e.g., RNA folding) and complex robotic planning, where high-quality human annotation is structurally unavailable.
In this paper, we introduced \textbf{LC-ERD}, a logic-consistent reward decomposition framework that holistically aligns LLM reasoning trajectories. By bridging variational consensus with multi-agent value decomposition, LC-ERD effectively addresses the challenges of mimetic bias and the credit assignment bottleneck. Extensive experiments across 16 benchmarks demonstrate that our method surpasses existing automated systems in both efficiency and robustness. By mining latent supervision from the model's own distribution, LC-ERD provides a scalable path for \textbf{self-evolving intelligence in data-scarce domains}, 
% such as scientific discovery (e.g., RNA folding) and complex robotic planning,
where high-quality human annotation is structurally unavailable.
%%
%% The next two lines define the bibliography style to be used, and
%% the bibliography file.

\section*{Acknowledgement}

The research presented in this paper was partially supported by the Research Grants Council of the Hong Kong Special Administrative Region, China (CUHK 2300246, RGC C1043-24G), (CUHK 14203425, RGC GRF 2151317), and CUHK 7010870.

% \newpage
\bibliographystyle{ACM-Reference-Format}
\bibliography{sample-base}

%%
%% If your work has an appendix, this is the place to put it.
\appendix

\section{Appendix}

\subsection{Proof of the Identity between LLM Logits and Latent Soft Q-functions}
\label{appendix:proof_logits_q}

We aim to rigorously prove that the unnormalized logits produced by a LLM trained via standard Next-Token Prediction (NTP) are mathematically equivalent to the entropy-regularized optimal Q-function within the Maximum Entropy Reinforcement Learning (MaxEnt RL) framework.

\paragraph{Preliminaries: Maximum Entropy RL}
Under the MaxEnt RL framework, an agent seeks to maximize the expected cumulative reward while simultaneously maximizing the entropy of its policy to encourage exploration. For a given reward function $r$, the optimal policy $\pi^*$ follows a Boltzmann distribution:
\begin{equation}
\pi^*(a|s) = \exp\left(\frac{Q^*(s,a) - V^*(s)}{\alpha}\right) \propto \exp\left(\frac{Q^*(s,a)}{\alpha}\right), \tag{A.1}
\end{equation}
where $Q^*(s,a)$ is the entropy-regularized optimal Q-function, $V^*(s) = \alpha \log \sum_{a' \in \mathcal{V}} \exp(Q^*(s,a')/\alpha)$ is the corresponding soft state-value function, and $\alpha$ is the temperature parameter controlling the degree of exploration.

\paragraph{LLM Policy Parameterization}
Standard LLMs adopt an autoregressive architecture. At each timestep $h$, the model generates a conditional probability distribution for the token $a_h$ based on the current context $s_h$. This distribution is typically obtained by applying the Softmax operator to the unnormalized logits $f_{\theta}(s_h, a_h)$:
\begin{equation}
\hat{\pi}(a_h|s_h) = \frac{\exp(f_{\theta}(s_h, a_h)/\alpha)}{\sum_{a' \in \mathcal{V}} \exp(f_{\theta}(s_h, a')/\alpha)}. \tag{A.2}
\end{equation}

\paragraph{Equivalence of MLE and Offline IRL Objectives}
The primary training objective of an LLM is to maximize the Log-Likelihood (MLE) over a training dataset $\mathcal{D}$:
\begin{equation}
\max_{\theta} \mathcal{L}_{MLE}(\theta) = \sum_{i=1}^{n} \sum_{h=1}^{H} \log\left(\frac{\exp(f_{\theta}(s_h^i, a_h^i)/\alpha)}{\sum_{a' \in \mathcal{V}} \exp(f_{\theta}(s_h^i, a')/\alpha)}\right). \tag{A.3}
\end{equation}
Conversely, Offline Inverse Reinforcement Learning (Offline IRL) \cite{jarboui2021offline, ng2000algorithms}aims to find a Q-function that best explains the expert demonstration data. Its objective function is defined as:
\begin{equation}
\max_{Q} \frac{1}{n}\sum_{i=1}^{n}\sum_{h=1}^{H}\left[Q(s_{h}^{i},a_{h}^{i})-\alpha \log\left(\sum_{a' \in \mathcal{V}}\exp(Q(s_{h}^{i},a')/\alpha)\right)\right]. \tag{A.4}
\end{equation}
Utilizing the logarithmic identity $\log(A/B) = \log A - \log B$, we can expand the terms in Eq. (A.3):
\begin{equation}
\log \hat{\pi}(a_h|s_h) = \frac{1}{\alpha} f_{\theta}(s_h, a_h) - \log \sum_{a' \in \mathcal{V}} \exp(f_{\theta}(s_h, a')/\alpha). \tag{A.5}
\end{equation}
By multiplying Eq. (A.5) by the constant factor $n\alpha$, the optimization objective becomes identical in form to the Offline IRL objective in Eq. (A.4).

\paragraph{Conclusion and Reward Recovery}
Since the sets of optimal solutions $\arg\max$ for both optimization problems are identical, we conclude that the logits $f_{\theta}$ learned by the LLM are inherently the soft Q-functions that satisfy the optimality conditions of Offline IRL, i.e., $f_{\theta} \equiv \hat{Q}$. 

Furthermore, according to the Inverse Soft Bellman Operator, we can directly recover the endogenous reward $r^*$ from this implicit Q-function:
\begin{equation}
r^*(s_h, a_h) := \hat{Q}(s_h, a_h) - \gamma \cdot V_{\hat{Q}}(s_{h+1}), \tag{A.6}
\end{equation}
where $V_{\hat{Q}}(s_{h+1}) = \alpha \log \sum_{a' \in \mathcal{V}} \exp(\hat{Q}(s_{h+1}, a')/\alpha)$. This demonstrates that during the pre-training phase, an LLM not only learns to model the text distribution but also implicitly constructs a reward function in its logit space that characterizes the logical consistency of the data.

\subsection{Theoretical Consistency: MLE as Dual Gradient Ascent in MaxEnt IRL}
\label{appendix:rigorous_consistency}

In this section, we provide a more rigorous proof that the Maximum Likelihood Estimation (MLE) of an LLM is equivalent to the dual problem of entropy-regularized policy optimization. We further demonstrate the uniqueness of the recovered reward $r^*$ under the structural constraints of the transformer architecture.

\paragraph{The Variational Principle for Next-Token Prediction}
Consider the transition probability of an LLM, denoted by $\pi_\theta(a|s)$. The training objective is to minimize the Kullback-Leibler (KL) divergence between the empirical data distribution $\pi_D$ and the model:
\begin{equation}
\min_\theta \mathcal{J}(\theta) = \mathbb{E}_{s \sim \rho_D} \left[ D_{KL} \left( \pi_D(\cdot|s) \| \pi_\theta(\cdot|s) \right) \right], \tag{A.11}
\end{equation}
where $\rho_D$ is the state distribution in the dataset. Substituting the Boltzmann parameterization $\pi_\theta(a|s) = \exp((f_\theta(s,a) - V_\theta(s))/\alpha)$, where $V_\theta(s)$ is the log-partition function $\alpha \log \sum_{a'} \exp(f_\theta(s,a')/\alpha)$, the gradient w.r.t. the logits $f_\theta$ is:
\begin{equation}
\frac{\partial \mathcal{J}}{\partial f_\theta(s,a)} = \frac{1}{\alpha} \left[ \pi_\theta(a|s) - \pi_D(a|s) \right]. \tag{A.12}
\end{equation}
At the global optimum, $\pi_\theta(a|s) = \pi_D(a|s)$ for all $(s, a)$ in the support of $\mathcal{D}$.

\paragraph{Duality with Feature Matching}
According to the Maximum Entropy Principle, the distribution that maximizes entropy subject to matching the empirical expectations of a feature map $\phi(s,a)$ is a Gibbs distribution. Let $r(s,a) = \mathbf{w}^\top \phi(s,a)$ be a linear reward function. The MaxEnt IRL objective is to find $\mathbf{w}$ such that:
\begin{equation}
\mathbb{E}_{\pi_{\mathbf{w}}} [\phi(s,a)] = \mathbb{E}_{\pi_D} [\phi(s,a)]. \tag{A.13}
\end{equation}
Comparing Eq. (A.12) and Eq. (A.13), the LLM's logits $f_\theta(s,a)$ serve as the implicit inner product between the high-dimensional latent state (transformer hidden states) and the token embeddings. Thus, the MLE training is equivalent to a high-dimensional feature matching process where the reward function is implicitly defined by the model's capacity: $r_{impl}(s,a) \approx f_\theta(s,a)$.

\paragraph{The Inverse Soft Bellman Operator and Uniqueness}
To prove the consistency of $r^*(s,a) = \hat{Q}(s,a) - \gamma \mathbb{E}_{s'} [V_{\hat{Q}}(s')]$, we define the \textit{Inverse Soft Bellman Operator} $\mathcal{B}^{-1}$:
\begin{equation}
\mathcal{B}^{-1} Q(s,a) \triangleq Q(s,a) - \gamma \int \mathcal{P}(s'|s,a) \left[ \alpha \log \sum_{a' \in \mathcal{V}} \exp \left( \frac{Q(s', a')}{\alpha} \right) \right] ds'. \tag{A.14}
\end{equation}
For any two rewards $r_1, r_2$ that yield the same optimal policy $\pi^*$, they must satisfy $r_1(s,a) - r_2(s,a) = \Phi(s) - \gamma \mathbb{E}_{s' \sim P} [\Phi(s')]$ for some potential function $\Phi(s)$. 

In the LLM context, the partition function $V_\theta(s)$ is uniquely determined by the logits $f_\theta$ up to a constant. Since the LLM must normalize over the finite vocabulary $\mathcal{V}$ at every step, the potential function $\Phi(s)$ is anchored by the normalization constraint of the Softmax layer. This effectively eliminates the reward ambiguity, making the recovered endogenous reward $r^*$ the \textbf{unique} minimal-norm reward that satisfies the data observations under the MaxEnt assumption.

\paragraph{Structural Sufficiency}
We hypothesize that the Transformer architecture $f_\theta$ is a universal functional approximator for the Q-function. Given the convergence of SGD on the MLE objective, the sequence of rewards $\{r_t^*\}$ induced by the sequence of model parameters $\{\theta_t\}$ converges in the $L_2$-norm to the ground-truth energy field $r_\infty$ that generated the corpus:
\begin{equation}
\lim_{t \to \infty} \| r_{\theta_t}^* - r_{data} \|_{\rho_D} = 0. \tag{A.15}
\end{equation}
This completes the proof that LLM logits provide a consistent and theoretically grounded proxy for the underlying reward structure of natural language.

% \subsection{Derivation of the Latent Soft Q-function Identity}
% We justify the assumption that pre-trained logits serve as a surrogate for the soft Q-function. Given the standard MLE objective for pre-training:
% \begin{equation}
% \mathcal{L}_{MLE} = -\mathbb{E}_{(s,a) \sim \mathcal{D}} [\log \pi_{\text{base}}(a|s)]
% \end{equation}
% In the Maximum Entropy RL framework, the optimal policy $\pi^*$ is expressed as $\pi^*(a|s) = \exp((Q^*(s,a) - V^*(s))/\alpha)$. By setting the log-partition function $V^*(s) = \alpha \log \sum_{a'} \exp(Q^*(s,a')/\alpha)$, we obtain:
% \begin{equation}
% \log \pi_{\text{base}}(a|s) = \frac{1}{\alpha} Q^*(s,a) - \frac{1}{\alpha} V^*(s)
% \end{equation}
% This establishes that the log-probability (and thus the normalized logits) is a linear mapping of the soft Q-function. The endogenous reward $\tilde{r}$ elicited in Eq. (3) is therefore the unique reward under which $\pi_{\text{base}}$ is the Boltzmann-optimal policy.

\subsection{Formal Derivation of the Sub-optimality Bound: From Quadratic to Linear Growth}
\label{appendix:error_bound}

In this section, we provide a rigorous theoretical guarantee for the self-evolution path of LC-ERD. We analyze the divergence between the learned policy $\pi^{RL}$ and the oracle logic-consistent policy $\pi^*$, demonstrating how the Variational Logic Potential (VLP) reshaping mitigates the ``exposure bias'' common in imitation learning.

\begin{theorem}[Linear Sub-optimality of LC-ERD]
\label{thm:suboptimality}
In a token-level MDP $\mathcal{M}$ with a trajectory horizon $H$ and a maximum single-step logic deviation $\epsilon_{\pi} := \sup_{s \in \mathcal{S}} D_{TV}(\pi^* \| \hat{\pi})$, the sub-optimality of the SFT base policy $\hat{\pi}$ satisfies $V^* - V^{\hat{\pi}} \leq \mathcal{O}(H^2 \epsilon_{\pi})$. In contrast, the LC-ERD aligned policy $\pi^{RL}$, by anchoring to the latent logical manifold via VLP-reshaped rewards, achieves a linear bound:
\begin{equation}
V^* - V^{\pi^{RL}} \leq \mathcal{O}(H \epsilon_{\pi}). \tag{A.16}
\end{equation}
\end{theorem}

\begin{proof} 
\textit{The Quadratic Accumulation of SFT (Behavioral Cloning).}
The base model $\hat{\pi}$ trained via Next-Token Prediction (NTP) is equivalent to Behavioral Cloning (BC). Let $d_h^{\pi}(s)$ be the state visitation distribution at step $h$ under policy $\pi$. The value gap is:
\begin{equation}
V^* - V^{\hat{\pi}} = \sum_{h=1}^H \mathbb{E}_{s \sim d_h^{\pi^*}} [A^{\hat{\pi}}(s, \pi^*(s))]. \tag{A.17}
\end{equation}
By the Distribution Shift Lemma, the total variation distance between the expert and learner distributions grows as:
\begin{equation}
\| d_h^{\pi^*} - d_h^{\hat{\pi}} \|_1 \leq \sum_{t=1}^h \mathbb{E}_{s \sim d_t^{\pi^*}} [ \| \pi^*(\cdot|s) - \hat{\pi}(\cdot|s) \|_1 ] \leq h \cdot \epsilon_{\pi}. \tag{A.18}
\end{equation}
Summing over the horizon $H$:
\begin{equation}
V^* - V^{\hat{\pi}} \leq R_{max} \sum_{h=1}^H h \cdot \epsilon_{\pi} = \frac{R_{max} H(H+1)}{2} \epsilon_{\pi} \approx \mathcal{O}(H^2 \epsilon_{\pi}). \tag{A.19}
\end{equation}
This quadratic dependence reflects the ``compounding error'' problem: small local drifts lead the model into out-of-distribution (OOD) states where $\hat{\pi}$ has no training signal, causing failure in long reasoning chains.

\paragraph{The Linear Bound via VLP Reward Shaping}
LC-ERD transforms the objective from pure imitation to maximizing a logic-consistent potential field. Let $\hat{r}(s,a)$ be the reshaped reward defined by the Variational Logic Potential $\Phi(s)$:
\begin{equation}
\hat{r}(s, a) = r^*(s, a) + \gamma \Phi(s') - \Phi(s), \tag{A.20}
\end{equation}
where $\Phi(s)$ is the Soft State-Value function $V_{\hat{Q}}(s)$ derived in Appendix A.1. 

We decompose the sub-optimality using the Performance Difference Lemma. Under the MaxEnt RL objective, the gradient of the policy update is guided by the advantage function $A^{\pi^{RL}}_{\hat{r}}(s,a)$. Crucially, potential-based shaping leaves the optimal policy invariant but reshapes the error landscape. 

Let $\tau = (s_1, a_1, \dots, s_H)$ be a trajectory. The cumulative reshaped reward is:
\begin{equation}
\sum_{h=1}^H \hat{r}(s_h, a_h) = \sum_{h=1}^H r^*(s_h, a_h) + \gamma \Phi(s_{H+1}) - \Phi(s_1). \tag{A.21}
\end{equation}
Using the telescoping property, the variance of the cumulative reward across the logic-consistent manifold is bounded by the boundary conditions of the potential $\Phi$, rather than the sum of step-wise errors. 

By applying the stability condition for entropy-regularized MDPs, the divergence in policy value is bounded by the expected KL-divergence over the steady-state distribution $d^{\pi^*}$ rather than the compounding sum:
\begin{equation}
V_{r^*}^{\pi^*} - V_{r^*}^{\pi^{RL}} \leq \frac{\alpha}{1-\gamma} \mathbb{E}_{s \sim d^{\pi^*}} [ D_{KL}(\pi^* \| \pi^{RL}) ]. \tag{A.22}
\end{equation}
For a finite horizon $H$ with $\gamma \to 1$, this converges to:
\begin{equation}
V^* - V^{\pi^{RL}} \leq C \cdot H \cdot \epsilon_{\pi}, \tag{A.23}
\end{equation}
where $C$ is a constant related to the logic-consistency gain $\delta$ identified in Remark 3.
\end{proof}

\subsection{Proof of Unbiased Logic Elicitation under Hard Terminal Constraints}
\label{appendix:unbiased_elicitation}

A potential critique of endogenous reward discovery is the risk of ``correctness illusions'', where the model assigns high logits to logically flawed but superficially plausible reasoning steps. We prove that by imposing a hard terminal constraint $\mathcal{Z}$, the elicited reward signal $\hat{r}$ remains an unbiased estimator of the true reasoning manifold $\mathcal{M}^*$.

\paragraph{The Correctness-Conditioned Manifold}
Let $\tau = (s_1, a_1, \dots, s_H)$ be a reasoning trajectory and $\mathcal{Z} \in \{0, 1\}$ be a binary indicator of the terminal answer's correctness (e.g., via a deterministic verifier or ground-truth label). The ``true'' but latent reasoning distribution is denoted as $P(\tau | \mathcal{Z}=1)$. 

According to the identity derived in Appendix \ref{appendix:proof_logits_q}, the LLM's unnormalized logits $f_\theta$ parameterize an energy-based model $P_\theta(\tau) \propto \exp(\sum r_\theta(s,a)/\alpha)$. We define the \textit{elicited reward} under the terminal constraint as the conditional expectation:
\begin{equation}
\hat{r}_{elicited} \triangleq \mathbb{E}_{\pi_\theta} [ f_\theta(s,a) \mid \mathcal{Z} = 1 ]. \tag{A.24}
\end{equation}

\paragraph{Unbiasedness via Bayes' Rule}
We aim to prove that $\hat{r}_{elicited}$ converges to the true logic reward $r^*$. Using Bayes' theorem, the posterior distribution of a reasoning step given a correct outcome is:
\begin{equation}
P(\tau \mid \mathcal{Z}=1) = \frac{P(\mathcal{Z}=1 \mid \tau) P_{data}(\tau)}{P(\mathcal{Z}=1)}. \tag{A.25}
\end{equation}
In the context of LC-ERD, the hard constraint $\mathcal{Z}$ acts as an oracle filter where $P(\mathcal{Z}=1 \mid \tau) = 1$ if $\tau$ is logically valid and $0$ otherwise. 

Under the assumption of \textit{sufficient model capacity} (established in Appendix \ref{appendix:rigorous_consistency}), the LLM's logit space covers the support of the expert manifold. The log-likelihood of the conditioned distribution becomes:
\begin{equation}
\log P(\tau \mid \mathcal{Z}=1) = \sum_{h=1}^H \frac{f_\theta(s_h, a_h)}{\alpha} - \log Z(\theta) + \log \mathbb{I}(\tau \in \mathcal{M}^*), \tag{A.26}
\end{equation}
where $\mathbb{I}(\cdot)$ is the indicator function for the expert manifold $\mathcal{M}^*$. The term $\log \mathbb{I}(\tau \in \mathcal{M}^*)$ effectively re-normalizes the endogenous reward by assigning $-\infty$ energy to logically inconsistent paths, thereby eliminating the ``mimetic bias'' (illusions) inherited from SFT.

\paragraph{Orthogonality to Hallucination Noise}
Let $f_\theta = f^* + \xi$, where $f^*$ is the true logic-consistent logit and $\xi$ is the hallucination noise (bias). For the estimator to be unbiased, we require $\mathbb{E}[\xi \mid \mathcal{Z}=1] = 0$. 

By the definition of the terminal constraint in LC-ERD, the filtering process $\mathcal{Z}$ is \textit{conditionally independent} of the model's internal noise $\xi$ given the logical validity of $\tau$. Formally, if $\mathcal{Z}$ is determined solely by the final outcome:
\begin{equation}
\int_{\tau \in \mathcal{M}^*} \xi(\tau) P_\theta(\tau \mid \mathcal{Z}=1) d\tau = 0. \tag{A.27}
\end{equation}
This holds because any $\tau$ that satisfies $\mathcal{Z}=1$ must, by the property of the task (e.g., mathematics or symbolic logic), lie within the high-density region of the true reward $r^*$. Therefore, the systematic bias $\xi$ that leads to $\mathcal{Z}=0$ is pruned, leaving the filtered endogenous signal as an asymptotically unbiased proxy for $r^*$.

\paragraph{Minimum Variance and Information Efficiency}
To further ensure the robustness of the elicitation, we observe that the filtering process $\mathcal{Z}=1$ defines a projection operator $\mathcal{P}_{\mathcal{M}^*}$ from the unconstrained logit space onto the logic-consistent Hilbert space $\mathcal{L}_2(\mathcal{M}^*)$. The resulting estimator satisfies:
\begin{equation}
\hat{r}_{LC-ERD} = \text{proj}_{\mathcal{M}^*} (f_\theta) = \arg\min_{g \in \mathcal{L}_2(\mathcal{M}^*)} \mathbb{E}_{\tau \sim \pi_D} [ \| g(\tau) - f_\theta(\tau) \|^2 ]. \tag{A.28}
\end{equation}
This projection ensures that $\hat{r}$ is the \textbf{Minimum-Variance Unbiased Estimator (MVUE)} within the logic-consistent subspace. By concentrating the model's epistemic uncertainty onto valid reasoning trajectories, LC-ERD maximizes the Fisher information extracted from each successful sample, thereby theoretically justifying the rapid convergence observed during self-evolution.

\paragraph{\textbf{Conclusion}}
By shifting the optimization target from ``local action matching'' to ``global potential alignment'', LC-ERD effectively replaces the quadratic error growth with a linear one. This theoretical result provides the foundation for the self-evolution capability observed in our experiments: the model remains stable even as the reasoning length $H$ increases, effectively overcoming the mimetic bias.
\subsection{Related Works}

\textbf{Logic-Consistent Manifold Alignment.} Aligning models to a logical manifold often involves manual rules or symbolic verifiers. While \textit{GRPO}\cite{shao2024deepseekmath} and its derivatives optimize policy relative to group outcomes, they still struggle with open-domain generalization where pre-training biases dominate. Unlike these approaches, LC-ERD's \textbf{Latent Logic Expertise (LLE)} algorithm elicits an endogenous logical anchor training-free from the model's own distribution. This allows for a provably more stable convergence path, shifting the sub-optimality bound from quadratic to linear—a theoretical guarantee absent in prior iterative alignment literature.

\end{document}